\documentclass[sigconf, nonacm]{acmart}

\newcommand\vldbdoi{XX.XX/XXX.XX}

\newcommand\vldbavailabilityurl{https://github.com/FangXieLab/TP-TopK}

\usepackage{mathtools}
\usepackage{amsthm}

\usepackage{algorithm}
\usepackage{algpseudocode}

\usepackage{booktabs}
\usepackage{tabularx}
\usepackage{multirow}
\usepackage{makecell}
\usepackage{adjustbox}
\newtheorem{theorem}{Theorem}[section]
\newtheorem{lemma}[theorem]{Lemma}
\newtheorem{corollary}[theorem]{Corollary}
\newtheorem{proposition}[theorem]{Proposition}
\newtheorem{definition}[theorem]{Definition}
\newtheorem{assumption}[theorem]{Assumption}
\newtheorem{remark}[theorem]{Remark}
\usepackage{enumitem}
\setlist[itemize]{topsep=3pt, itemsep=2pt, parsep=0pt, partopsep=0pt}
\setlist[enumerate]{topsep=3pt, itemsep=2pt, parsep=0pt, partopsep=0pt}

\begin{document}
\title{When Do Fewer Coordinates Suffice in DP-SGD?}

\author{Huiqi Zhang}
\affiliation{%
  \institution{Guangdong Provincial Key Laboratory of IRADS, Beijing Normal-Hong Kong Baptist University}
  \streetaddress{P.O. Box 1212}
  \city{Zhuhai}
  \state{China}
  \postcode{519087}
}

\author{Fang Xie}
\orcid{}
\affiliation{%
  \institution{Guangdong Provincial Key Laboratory of IRADS, Beijing Normal-Hong Kong Baptist University}
  \streetaddress{P.O. Box 1212}
  \city{Zhuhai}
  \state{China}
  \postcode{519087}
}

\begin{abstract}
Differentially private stochastic gradient descent (DP-SGD) injects
noise into every updated coordinate, making the injected noise energy
scale with the ambient parameter dimension \(d\). We ask when private
training can update fewer coordinates without losing the signal needed
for optimization. We propose \textsc{TP-TopK} (Two-Phase TopK DP-SGD),
a two-phase method for coordinate-sparse private training without public
data, in which a private warm-up phase identifies a coordinate support
used to guide the main training phase. We give a criterion
characterizing when coordinate restriction can be beneficial, show via a
nonconvex stationarity bound that under this condition the relevant noise
term scales with the active dimension \(k\) rather than the full
parameter dimension \(d\), and provide a lower bound on the reliability
of warm-up-based coordinate ranking. Experiments on MNIST,
FMNIST, and CIFAR-10 show that learned coordinate supports can
retain more gradient energy than size-matched random supports, with the
largest gains when the active dimension is small and warm-up scores are
informative.
\end{abstract}

\maketitle


\ifdefempty{\vldbavailabilityurl}{}{
	\vspace{.3cm}
	\begingroup\small\noindent\raggedright\textbf{Artifact Availability:}\\
	The source code, data, and/or other artifacts have been made available at \url{\vldbavailabilityurl}.
	\endgroup
}

\section{Introduction}
\label{sec:introduction}
Deep learning has achieved remarkable success in tasks such as image recognition
~\cite{Lundervold2018AnOO, Zhou2020ARO, dosovitskiy2020vit, shamshad2023transformers},
text analysis~\cite{gmail2019, brown2020language, zhao2023survey},
and recommendation systems~\cite{Lu2012RecommenderS, wu2023survey},
largely due to centralized training on large-scale datasets.

However, this paradigm raises significant privacy concerns, as user data is 
exposed to leakage risks during both training and deployment---including 
membership inference attacks~\cite{shokri2017membership, hu2022membership, carlini2022membership} and 
gradient inversion attacks~\cite{zhu2019deep, yin2021see}.
Differential Privacy (DP)~\cite{dwork2006differential} provides formal privacy
guarantees by injecting calibrated noise to obscure the influence of any
individual data point. Since DP was introduced to deep learning~\cite{shokri2015privacy},
Differentially Private Stochastic Gradient Descent
(DP-SGD)~\cite{abadi2016deep} has emerged as the predominant method: per-sample
gradients are clipped and aggregated, followed by Gaussian noise injection to
satisfy formal $(\varepsilon, \delta)$-DP guarantees. Nevertheless, DP-SGD
suffers from inherent limitations. First, it injects uniform noise across all
gradient dimensions, and the noise magnitude grows with model size, impairing
convergence~\cite{bassily2014private}. Second, the privacy budget accumulates
linearly over long training runs, further degrading model utility. Third, noise
is indiscriminately added to low-importance parameters, weakening the signal
carried by critical gradients~\cite{chen2023differentially}.

The root cause of these limitations is the dimensionality dependence of the
injected noise. For isotropic Gaussian perturbation
$z \sim \mathcal{N}(0, \sigma^2 C^2 I_d)$, we have
$\mathbb{E}\|z\|_2^2 = d\sigma^2 C^2$. As the parameter dimension $d$ grows, this scaling directly degrades utility in high-dimensional deep models.

We study the \emph{pure private-data} setting — no public data, no pretrained checkpoint, training from scratch — where this dimension dependence is the central obstacle. Public-data and pretrained pipelines are powerful, but they assume public
representations that are both privacy-appropriate and aligned with the private
task; this assumption routinely fails in sensitive, domain-specific settings such as
clinical records, financial risk modeling, rare-disease analysis, and private
enterprise data~\citep{tramer2024considerations,hod2025surrogate}. Empirical evidence from medical imaging further confirms that 
DP noise disproportionately affects model utility in sensitive 
domains~\citep{mohammadi2026dp}, motivating methods that 
concentrate the privacy cost on the most informative parameters. We therefore ask whether noisy gradients from private training carry 
sufficient signal to identify a high-energy coordinate support. 
Figure~\ref{fig:motivation} answers affirmatively: gradient energy 
is sharply concentrated in a small coordinate subset~(a), and 
\textsc{TP-TopK} warm-up scores reliably recover high-energy supports 
even at $\varepsilon = 1$, evaluated against the oracle gradient energy~(b). 
This motivates \textsc{TP-TopK}: a sparse private training mechanism 
that restricts both optimization and noise injection to the learned 
support, reducing the effective noise dimension from $d$ to $k \ll d$.

\begin{figure*}[t]
    \centering
    \adjustbox{trim={0 0.50\height{} 0 0},clip}{
        \includegraphics[width=\textwidth]{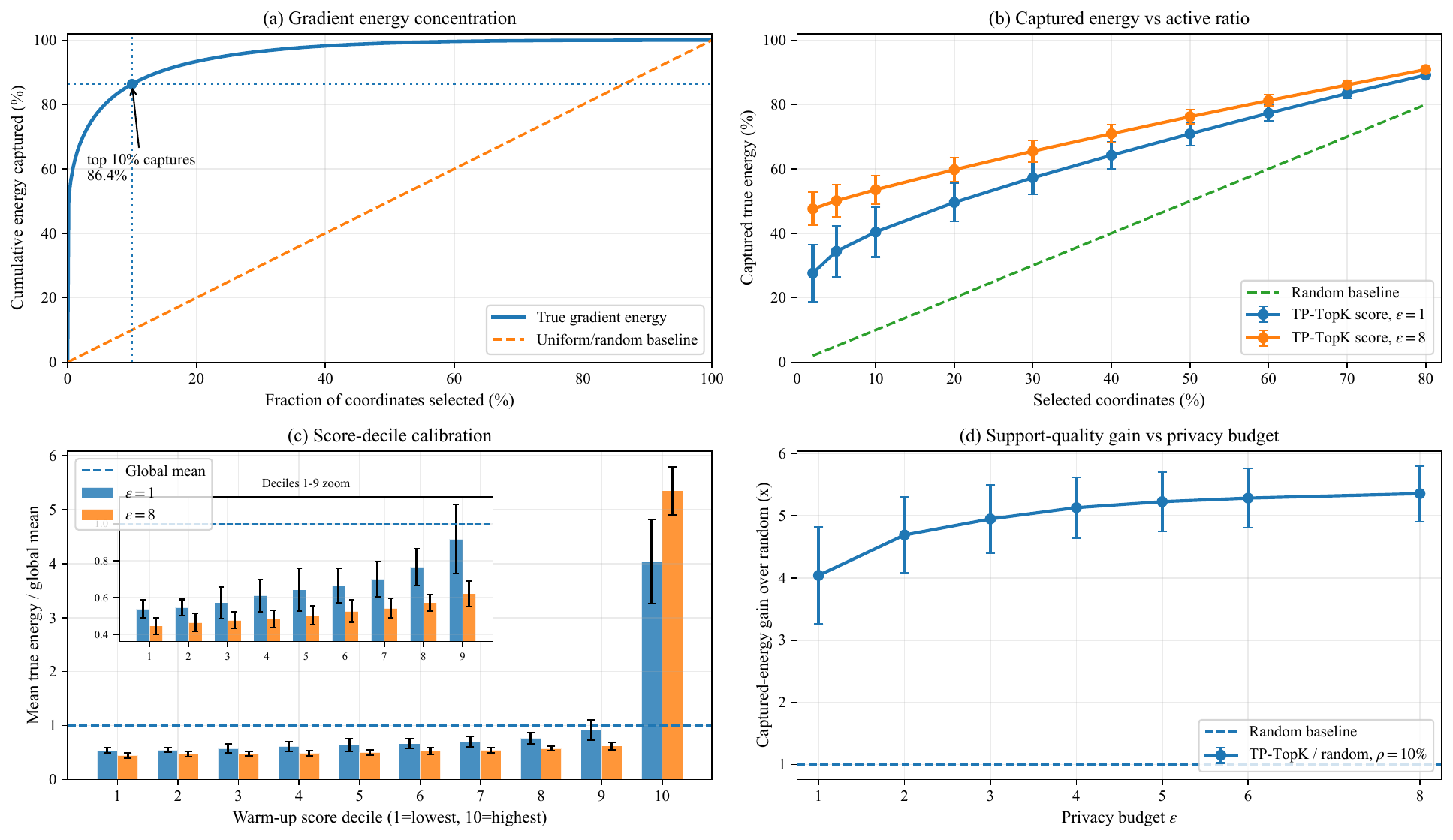}
    }
    \caption{
    Empirical motivation for \textsc{TP-TopK}. Gradient energy is 
defined as the per-coordinate squared gradient magnitude $E_j = g_j^2$, 
averaged over non-private mini-batches as a post-hoc oracle diagnostic.
(a)~The top $10\%$ of coordinates capture $86.4\%$ of total gradient energy, 
far exceeding a random selection of equal size.
(b)~Supports selected by \textsc{TP-TopK} warm-up scores recover 
substantially more gradient energy than random supports of the same 
size, even at $\varepsilon = 1$.
    }
    \label{fig:motivation}
\end{figure*}

The core contributions of this work are as follows:
\begin{itemize}
    \item We propose \textsc{TP-TopK}, a two-phase coordinate-sparse
    differentially private training framework. A short DP warm-up phase
    scores each coordinate using privatized gradient aggregates and
    selects a support set; Phase~2 then runs DP-SGD restricted to the
    support. This reduces the per-step optimizer-facing DP noise energy from
\(d\nu_2^2\) to \(k\nu_2^2\), where
\(\nu_2^2=(\sigma_2 C_2/B)^2\) is the per-coordinate variance of the
averaged Phase~2 Gaussian noise and \(k \ll d\).

    \item We provide three theoretical results: (1) a proof that
    \textsc{TP-TopK} satisfies $(\varepsilon, \delta)$-differential
    privacy; (2) a smooth nonconvex convergence bound in which the
    DP-noise term scales with active dimension $k$ rather than full
    dimension $d$; and (3) a one-step criterion characterizing when
    coordinate restriction improves utility, together with a misranking
    probability bound identifying conditions under which the learned
    support reliably outperforms random selection.

    \item We evaluate \textsc{TP-TopK} on MNIST, FMNIST, and
CIFAR-10 through a controlled five-method comparison
(\textsc{TP-TopK}, \textsc{TP-Rand}, ON-TopK,
ON-Rand~\cite{zhu2023improving}, full DP-SGD), with all
methods matched on active ratio, training budget, and privacy
accounting to isolate the contribution of coordinate selection
quality. We further validate our method on EyePACS diabetic
retinopathy screening~\citep{eyepacs}, a sensitive medical
domain where public proxy data is unavailable, demonstrating
recovery of minority-class recall that collapses under full
DP-SGD.
\end{itemize}

\section{Related Work}
\label{sec:related-work}

\paragraph{Differentially private optimization.}
\citet{dwork2006differential} introduced differential privacy as a
formal guarantee that any single example has bounded influence on an
algorithm's output. Early work on private SGD includes
\citet{song2013stochastic} and \citet{bassily2014private};
\citet{abadi2016deep} adapted this framework to deep learning through
DP-SGD, which combines per-example gradient clipping with Gaussian
noise injection. Subsequent work has improved privacy accounting
\citep{mironov2017renyi,wang2019subsampled} and clipping strategies
\citep{andrew2021differentially,bu2023automatic}. However, the utility of
DP-SGD remains sensitive to the dimension of the optimized parameter space:
for isotropic Gaussian perturbations, the Euclidean noise magnitude scales as
\(\sigma C\sqrt{d}\), making high-dimensional private optimization difficult
\citep{bassily2014private}. Architectural and hyperparameter choices can
partly mitigate this issue \citep{de2022unlocking}.

\paragraph{Noise mitigation in a fixed parameter space.}
Several methods reduce the effect of DP noise while keeping the full
parameter space active. DOPPLER applies frequency-domain filtering to
attenuate high-frequency DP noise while preserving gradient signal
\citep{zhang2024doppler}. DiSK uses a simplified Kalman filter to
denoise sequences of privatized updates \citep{zhang2025disk}. DPDR
decomposes gradients into reusable and incremental components and
allocates more privacy budget to the latter \citep{liu2024dpdr}.
These methods are complementary to ours: they aim to recover useful
signal after privatization in the original parameter space, whereas
our method reduces the dimension of the space in which Phase~2 noise
is injected. A complementary direction, orthogonal to both noise-filtering approaches and our method, replaces isotropic Gaussian noise with structured correlated noise via matrix mechanisms
\citep{kairouz2021practical, choquette2023amplified,
choquette2025nearexact}.

\paragraph{Subspace, sparse, and selective private training.}
A closer line of work restricts private updates to a lower-dimensional
subspace or parameter subset. Public-data subspace methods use
auxiliary data to identify a low-dimensional gradient subspace for
private optimization \citep{zhou2021bypassing, yu2021gep, yu2021large},
while low-rank reparameterization reduces the number of trainable
degrees of freedom \citep{yu2021large}. Other methods sparsify
gradients before noise addition \citep{zhu2023improving} or select
and freeze parameters using public data
\citep{adamczewski2023differential}. Theoretical work has studied the
convergence implications of gradient clipping under stochastic bias
\citep{koloskova2023revisiting} and the benefits of dimension
reduction under gradient sparsity \citep{ghazi2023sparsity}. SPARTA 
selects a sparse fine-tuning mask via private gradient 
information~\citep{makni2025sparta}, but targets fine-tuning from a 
public pretrained checkpoint---a setting that \citet{tramer2024considerations} 
argue may be unavailable or privacy-inappropriate in sensitive domains. 
Recent work has also highlighted that standard DP benchmarks may not 
reflect the challenges of sensitive deployment 
settings~\citep{mokhtari2026rethinking}, and that sparse model structures 
can improve the signal-to-noise ratio under DP 
constraints~\citep{benz2023equivariant}. \textsc{TP-TopK} is designed 
to address precisely these concerns: it trains from scratch without any 
public data, directly targeting the sensitive deployment settings where 
existing methods fail.

\section{Our Method}
\label{sec:Method}

\subsection{Problem Setting}
\label{sec:problem-setting}

Let \(D=\{(x_i,y_i)\}_{i=1}^N\) denote a private training set, where
\(x_i\) is an input and \(y_i\) is its label. Let
\(\theta\in\mathbb{R}^d\) be the vector of all trainable scalar parameters,
and let \(\ell(\theta;x_i,y_i)\) be the per-example loss. We consider
empirical risk minimization with objective
\[
  L_D(\theta)
  =
  \frac{1}{N}\sum_{i=1}^N \ell(\theta;x_i,y_i).
\]
The learning algorithm must output a private model \(\hat\theta\) that
retains high test performance while satisfying
\((\varepsilon,\delta)\)-differential privacy
\citep{dwork2014algorithmic,abadi2016deep}; informally, changing one
training example should have only a limited effect on the output
distribution.

At iteration \(t\), DP-SGD samples a mini-batch index set
\(\mathcal{B}_t\subseteq\{1,\ldots,N\}\) of size \(B\), computes per-example gradients
\[
  g_{t,i}
  =
  \nabla_\theta \ell(\theta_t;x_i,y_i),
  \qquad i\in\mathcal{B}_t,
\]
and clips each gradient to an \(\ell_2\)-norm threshold \(C\):
\[
  \bar{g}_{t,i}
  =
  g_{t,i}
   \cdot\min\!\left\{1,\frac{C}{\|g_{t,i}\|_2}\right\}.
\]
The clipped gradients are aggregated and perturbed with Gaussian
noise,
\[
  \tilde{g}_t
  =
  \frac{1}{B}
  \!\left(
    \sum_{i\in\mathcal{B}_t}\bar{g}_{t,i}
    +
    z_t
  \right),
  \qquad
  z_t\sim\mathcal{N}(0,\sigma^2C^2I_d),
\]
and the model is updated by
\[
  \theta_{t+1}
  =
  \theta_t-\eta_t\tilde{g}_t,
\]
where \(\eta_t>0\) is the learning rate at step \(t\).

We track cumulative privacy loss using subsampled R\'{e}nyi differential
privacy (RDP) accounting \citep{mironov2017renyi,wang2019subsampled, gopi2021numerical} and
convert the resulting guarantee to \((\varepsilon,\delta)\)-DP. 

The privacy loss of DP-SGD is governed by the sampling rate \(B/N\), the
noise multiplier \(\sigma\), the number of training steps \(T\), and
\(\delta\). However, the magnitude of the injected noise seen by the
optimizer depends on the ambient parameter dimension. For
\(z_t\sim\mathcal{N}(0,\sigma^2C^2I_d)\),
\[
  \mathbb{E}
  \left\|\frac{1}{B}z_t\right\|_2^2
  =
  \frac{d\sigma^2C^2}{B^2}.
\]
Thus, for fixed \(\sigma\), \(B/N\), \(T\), and \(\delta\), increasing \(d\)
does not increase the privacy loss, but it does increase
\(\mathbb{E}\|\frac{1}{B}z_t\|_2^2\) proportionally. This dimension
dependence motivates reducing the effective coordinate dimension in which
DP noise is injected.

We represent a coordinate-sparse training configuration by an active
coordinate support $A\subseteq\{1,\ldots,d\}$, or equivalently by a
binary mask $m\in\{0,1\}^d$ with $m_p=1[p\in A]$ for $p\in\{1,\ldots,d\}$. We define
\[
  k=|A|,
  \qquad
  \rho(A)=\frac{k}{d},
  \qquad
  \bar{A}=\{1,\ldots,d\}\setminus A,
\]
where \(\rho(A)\) is the active ratio, i.e., the fraction of coordinates
updated in Phase~2, and \(\bar{A}\) is the frozen complement.

In our two-phase method, Phase~1 produces a coordinate score vector
\(a\in\mathbb{R}^d\), and the active support is selected as
\(A=\operatorname{TopK}(a,k)\), the set of \(k\) coordinates with the
largest scores. When discussing random-support baselines or oracle
diagnostics, we explicitly write \(A_{\mathrm{rand}}\) and \(A^\star\),
respectively.

\subsection{Two-Phase TopK DP-SGD}
\label{sec:method}
A smaller active support reduces the dimension-dependent DP noise seen
by the optimizer, but introduces projection bias by discarding gradient
signal outside $A$, namely on $\bar A$. The central design question is
how to choose $A$ so that the noise reduction outweighs the signal loss.

We propose \textsc{TP-TopK}, a two-phase coordinate-sparse private training
procedure: a short DP warm-up phase discovers a coordinate support from
private gradient statistics, and a second DP-SGD phase trains only on the
selected support. Figure~\ref{fig:two-phase-pipeline} illustrates the full
two-phase training pipeline and the corresponding RDP accounting structure.
Algorithm~\ref{alg:TP-TopK} summarizes the full procedure.

\begin{figure}[t]
\centering
\includegraphics[width=\linewidth]{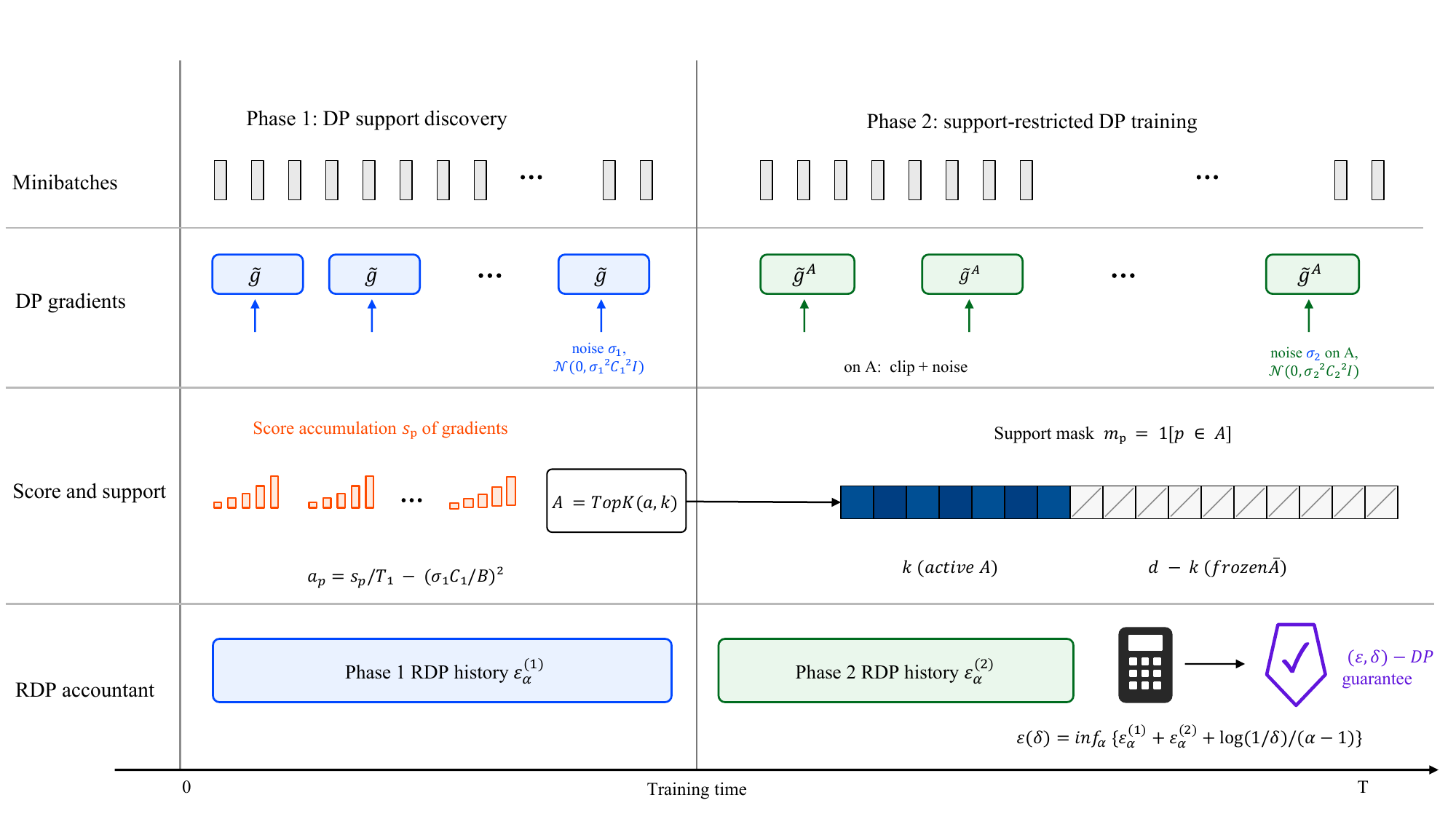}
\caption{
Two-phase training pipeline of \textsc{TP-TopK} and its RDP accounting
structure. Phase~1 runs full-parameter DP-SGD with noise multiplier
$\sigma_1$, accumulating coordinate scores $s_p$; the denoised score
$a_p = s_p/T_1 - (\sigma_1 C_1/B)^2$ is post-processed to select
$A = \operatorname{TopK}(a,k)$ at no additional privacy cost.
Phase~2 restricts DP-SGD to the $k$ active coordinates in $A$ with
noise multiplier $\sigma_2$, leaving frozen coordinates $\bar{A}$
unchanged. The two RDP histories compose additively to yield the
end-to-end $(\varepsilon,\delta)$-DP guarantee of
Theorem~\ref{thm:privacy}.
}
\label{fig:two-phase-pipeline}
\end{figure}

\begin{algorithm}[t]
\small
\caption{\textsc{TP-TopK}: Two-Phase TopK DP-SGD}
\label{alg:TP-TopK}
\begin{algorithmic}[1]
\Require Dataset $D$; initial parameters $\theta_0 \in \mathbb{R}^d$;
  steps $T_1, T_2$; batch size $B$;
  clipping norms $C_1, C_2$; noise multipliers $\sigma_1, \sigma_2$;
  learning rates $\eta_t$; active ratio $\rho$.
\Ensure Final private model $\theta$.

\Statex \textbf{Phase 1: DP warm-up guided support discovery}

\State $\theta \gets \theta_0$,\quad $s \gets \mathbf{0} \in \mathbb{R}^d$
\For{$t = 1, \ldots, T_1$}
  \State Sample $\mathcal{B} \subseteq \{1,\ldots,N\}$ of size $B$
  \For{$i \in \mathcal{B}$}
    \State $\bar{g}_i \gets g_i \cdot\min\!\left(1, C_1/\|g_i\|_2\right)$
      where $g_i = \nabla_\theta \ell(\theta; x_i, y_i)$
  \EndFor
  \State $\tilde{g} \gets \frac{1}{B}\!\left(\sum_{i \in \mathcal{B}} \bar{g}_i + z\right)$,\quad $z \sim \mathcal{N}(0, \sigma_1^2 C_1^2 I_d)$
  \State $\theta \gets \theta - \eta_t \tilde{g}$,\quad $s \gets s + \tilde{g}^2$
\EndFor
\State $\theta^{(1)} \gets \theta$
\State $a_p \gets s_p / T_1 - (\sigma_1 C_1 / B)^2$ for $p \in \{1,\ldots,d\}$
\State $k \gets \lfloor \rho d \rfloor$,\quad $A \gets \operatorname{TopK}(a, k)$,\quad $m_p \gets \mathbf{1}[p \in A]$

\Statex
\Statex \textbf{Phase 2: Support-restricted masked DP-SGD}

\State $\theta \gets \theta^{(1)}$
\For{$t = 1, \ldots, T_2$}
  \State Sample $\mathcal{B} \subseteq \{1,\ldots,N\}$ of size $B$
  \For{$i \in \mathcal{B}$}
    \State $\bar{g}^A_i \gets (m \odot g_i) \cdot \min\!\left(1, C_2 / \|m \odot g_i\|_2\right)$
  \EndFor
  \State $\widetilde{g}^A \gets \frac{1}{B}\!\left(\sum_{i \in \mathcal{B}} \bar{g}^A_i + m \odot z\right)$,\quad $z \sim \mathcal{N}(0, \sigma_2^2 C_2^2 I_d)$
  \State $\theta \gets \theta - \eta_t\, \widetilde{g}^A$
\EndFor
\State \Return $\theta$
\end{algorithmic}
\end{algorithm}

\paragraph{Phase 1: DP warm-up guided support discovery.}
We run full-parameter DP-SGD for \(T_1\) steps. At each step, per-example
gradients are clipped to norm \(C_1\), aggregated, and perturbed with
isotropic Gaussian noise to produce the averaged privatized gradient
\(\tilde{g}_t\).
We score each coordinate by its denoised squared privatized gradient
accumulated over the warm-up trajectory. The score for coordinate \(p\) is
\[
  a_p
  =
  \frac{1}{T_1}\sum_{t=1}^{T_1}\tilde{g}_{t,p}^{\,2}
  -
  \left(\frac{\sigma_1 C_1}{B}\right)^2,
\]
where the subtracted term is the per-coordinate DP noise variance. Thus $a_p$ is a noise-floor-corrected estimate of the coordinate-wise 
squared gradient signal; scores may be negative for low-signal coordinates, 
but this is not problematic as they are used only as ranking values. 
See Lemma~\ref{lem:score-unbiased} for unbiasedness under the simplified 
additive model, and Proposition~\ref{prop:rankability} for the resulting 
ranking reliability. Because each \(\tilde{g}_t\) is already a
DP output, score computation and support selection are post-processing and
incur no additional privacy cost. The active support is
\(A=\operatorname{TopK}(a,k)\) with binary mask \(m_p=1[p\in A]\).

\paragraph{Phase 2: support-restricted masked DP-SGD.}
Phase~2 starts from the warm-up checkpoint \(\theta^{(1)}\) and
keeps \(A\) fixed throughout. At each step, the per-example gradient is masked to the active support, $g^A_{t,i} = m\odot g_{t,i},$
then clipped to norm \(C_2\),
\[
  \bar{g}^A_{t,i}
  =
  g^A_{t,i}
  \cdot
  \min\!\left\{1,\frac{C_2}{\|g^A_{t,i}\|_2}\right\}.
\]
The clipped gradients are aggregated and Gaussian noise is injected only on
the active coordinates,
\[
  \tilde{g}_t^A
  =
  \frac{1}{B}
  \!\left(
    \sum_{i\in\mathcal{B}_t}\bar{g}^A_{t,i}
    +
    m\odot z_t
  \right),
  \qquad
  z_t\sim\mathcal{N}(0,\sigma_2^2C_2^2I_d),
\]
and the model is updated by
\[
  \theta_{t+1}
  =
  \theta_t - \eta_t\,\tilde{g}_t^A.
\]
Coordinates in \(\bar{A}\) receive neither gradient updates nor added noise.
Equivalently, Phase~2 applies the Gaussian mechanism on the
\(k\)-dimensional active subspace and embeds the result back into
\(\mathbb{R}^d\) deterministically, so the standard subsampled Gaussian RDP
accountant applies directly. The averaged noise energy is
\[
  \mathbb{E}
  \left\|\frac{1}{B}m\odot z_t\right\|_2^2
  =
  \frac{k\sigma_2^2C_2^2}{B^2},
\]
compared with \(d\sigma_2^2C_2^2/B^2\) for full-parameter DP-SGD. Phase~2
therefore reduces the optimizer-facing DP noise energy by the active ratio
\(\rho(A)=k/d\).

The end-to-end privacy guarantee for both phases is analyzed in 
Section~\ref{sec:privacy-analysis}.

\section{Privacy Analysis}
\label{sec:privacy-analysis}

The algorithm releases quantities from two phases.
Phase~1 releases the privatized transcript $\mathcal{T}_1$;
Phase~2 queries $D$ again through a support-restricted
DP-SGD mechanism and is accounted for separately.
The total cost follows from adaptive RDP composition
\citep{mironov2017renyi}.

\begin{theorem}[End-to-end privacy guarantee]
\label{thm:privacy}
Fix $\delta \in (0,1)$.  Suppose Phase~1 runs $T_1$ steps of the
Poisson-subsampled Gaussian mechanism with sampling rate $q_1$,
clipping norm $C_1$, and noise multiplier $\sigma_1$, incurring
RDP cost $\varepsilon_\alpha^{(1)}$ at order $\alpha > 1$.
Let $\theta^{(1)}$, $a$, $A$, and $m$ be computed from
the Phase~1 transcript as post-processing, incurring no additional
privacy cost.  Suppose Phase~2 runs $T_2$ steps of
support-restricted DP-SGD with sampling rate $q_2$, clipping norm
$C_2$, and noise multiplier $\sigma_2$, incurring RDP cost
$\varepsilon_\alpha^{(2)}$ at order $\alpha$.  Then the two-phase
mechanism satisfies $(\varepsilon(\delta),\delta)$-DP with
\[
  \varepsilon(\delta)
  =
  \inf_{\alpha > 1}
  \left\{
    \varepsilon_\alpha^{(1)}
    + \varepsilon_\alpha^{(2)}
    + \frac{\log(1/\delta)}{\alpha - 1}
  \right\}.
\]
\end{theorem}

Explicit per-phase RDP costs and the practical PRV accountant
are given in Remark~\ref{rem:rdp-explicit}.
The full proof is given in Appendix~\ref{app:privacy-proof}.

\paragraph{Proof sketch.}
\emph{Phase~1.}  Each step clips per-example gradients to $C_1$
and adds Gaussian noise with multiplier $\sigma_1$.  By the
subsampled Gaussian RDP accountant \citep{mironov2017renyi},
$T_1$ steps satisfy $(\alpha,\varepsilon_\alpha^{(1)})$-RDP.

\emph{Post-processing (zero additional cost).}
$\theta^{(1)}$, $a$, $A$, and $m$ are deterministic
functions of $\mathcal{T}_1$ and do not access $D$
beyond what is already encoded in $\mathcal{T}_1$.
By post-processing immunity \citep{dwork2014algorithmic}, they
incur no additional privacy cost.

\emph{Phase~2.} Fix any Phase~1 transcript, so $A$ is determined.  
The key observation is that masking precedes clipping: the 
per-example sensitivity is $C_2$ regardless of the support size $k$, 
so the standard subsampled Gaussian RDP accountant applies on the 
$k$-dimensional active subspace. Adaptive RDP composition then gives 
the stated bound.

\section{Support-Quality Analysis}
\label{sec:support-quality-analysis}

We analyze when coordinate-sparse private training improves over
full-parameter DP-SGD. The argument proceeds in two steps:
a one-step proxy condition characterizing when coordinate restriction
is beneficial, and an analysis of when the Phase~1 DP warm-up scores
reliably identify the beneficial coordinates.

\paragraph{When is coordinate restriction beneficial?}
Fix a parameter vector $\theta$. Using the notation of
Section~\ref{sec:problem-setting}, let
\[
  \bar{g}_i
  =
  \nabla_\theta \ell(\theta; x_i, y_i)
  \cdot
  \min\!\left\{1,\frac{C_2}{\|\nabla_\theta \ell(\theta; x_i, y_i)\|_2}\right\}
\]
be the clipped per-example gradient at $\theta$ with Phase~2 clipping
norm $C_2$, and let
\[
  G = \frac{1}{N}\sum_{i=1}^{N} \bar{g}_i \in \mathbb{R}^d
\]
denote the full-data clipped gradient aggregate---the full-data,
noiseless analog of the per-step mini-batch aggregate in
Section~\ref{sec:method}.
Let $P_{\bar{A}}$ denote the coordinate projection onto the
inactive coordinates $\bar{A}$, and define the per-coordinate
Phase~2 DP noise variance as $\nu_2^2 = \left(\frac{\sigma_2 C_2}{B}\right)^2$.

\begin{proposition}[One-step coordinate-restriction criterion]
\label{prop:one-step-criterion}
The one-step squared error of a support-restricted update decomposes as
\[
  \mathcal{E}(A)
  =
  \underbrace{\|P_{\bar{A}}G\|_2^2}_{\text{signal loss}}
  +
  \underbrace{k\nu_2^2}_{\text{active-coord.\ DP noise}},
\]
while full-parameter DP-SGD has proxy error
$\mathcal{E}_{\mathrm{full}} = d\nu_2^2$. This follows by setting $A = \{1,\ldots,d\}$, giving $P_{\bar{A}}G = 0$ and noise term $d\nu_2^2$.
The condition $\mathcal{E}(A) < \mathcal{E}_{\mathrm{full}}$ reduces to
\[
  \|P_{\bar{A}}G\|_2^2 < (d-k)\nu_2^2,
\]
since $\mathcal{E}(A) - \mathcal{E}_{\mathrm{full}}
= \|P_{\bar{A}}G\|_2^2 - (d-k)\nu_2^2$.
\end{proposition}

Support restriction therefore helps when the discarded gradient
energy is smaller than the DP noise energy saved by freezing
the inactive coordinates. This is a one-step, single-iterate
proxy condition; the multi-step analysis is given in
Theorem~\ref{thm:projected-stationarity}.
This decomposition parallels the
signal-loss/perturbation-moderation trade-off of
\citet{zhu2023improving}; our question is when a \emph{learned}
support makes the trade-off favorable.

A random support $A_{\mathrm{rand}}$ of size $k = \rho d$ drawn
uniformly over all $\binom{d}{k}$ subsets has expected signal loss
$(1-\rho)\|G\|_2^2$ by symmetry.
A learned support improves over random when
$\|P_{\bar{A}}G\|_2^2 < (1-\rho)\|G\|_2^2$. Since $G$ is
unobservable, \textsc{TP-TopK} uses the Phase~1 denoised score
$a_p$ as a DP-visible proxy for $G_p^2$: by
Lemma~\ref{lem:score-unbiased}, $\mathbb{E}[a_p] = G_p^2$ under
the simplified model, so $\operatorname{TopK}(a,k)$ selects
coordinates with the highest estimated gradient energy. The
reliability of this proxy ranking is analyzed next.

\paragraph{Reliability of DP warm-up scores.}

Under a simplified additive noise model, the privatized warm-up
gradient at step $t$ satisfies
\[
  \tilde{g}_{t,p} = G_p + \frac{1}{B}z_{t,p},
  \qquad
  z_{t,p} \overset{\mathrm{i.i.d.}}{\sim} \mathcal{N}(0,\,\sigma_1^2 C_1^2),
\]
where $G_p$ is the $p$-th coordinate of the fixed full-data clipped
gradient $G$, $z_{t,p}$ is the $p$-th coordinate of the Phase~1
noise vector $z_t \sim \mathcal{N}(0,\sigma_1^2 C_1^2 I_d)$ from
Section~\ref{sec:problem-setting}, and $\nu_1^2 = \left(\frac{\sigma_1 C_1}{B}\right)^2$
is the per-coordinate noise variance in the averaged gradient.
This model treats $G$ as fixed across warm-up steps and isolates
the effect of DP noise on coordinate ranking.

\begin{lemma}[Unbiasedness of denoised coordinate scores]
\label{lem:score-unbiased}
Under the additive noise model above, the Phase~1 score
\[
  a_p = \frac{1}{T_1}\sum_{t=1}^{T_1}
  \bigl(\tilde{g}_{t,p}^{\,2} - \nu_1^2\bigr)
\]
satisfies $\mathbb{E}[a_p] = G_p^2$ for each coordinate $p$.
\end{lemma}

\begin{proof}
For each step $t$, write $\tilde{g}_{t,p} = G_p + \frac{1}{B}z_{t,p}$
with $\mathbb{E}[z_{t,p}] = 0$ and
$\mathbb{E}[z_{t,p}^2] = \sigma_1^2 C_1^2$. Then
\[
  \mathbb{E}[\tilde{g}_{t,p}^2]
  = G_p^2
    + \frac{2G_p}{B}\underbrace{\mathbb{E}[z_{t,p}]}_{=\,0}
    + \frac{1}{B^2}\underbrace{\mathbb{E}[z_{t,p}^2]}_{=\,\sigma_1^2 C_1^2}
  = G_p^2 + \nu_1^2.
\]
Subtracting $\nu_1^2$ and averaging over $t = 1,\ldots,T_1$ gives
$\mathbb{E}[a_p] = G_p^2$ by linearity of expectation.
\end{proof}

\begin{proposition}[Misranking probability of DP warm-up scores]
\label{prop:rankability}
Under the model above, if $G_p^2 > G_q^2$, then
\[
  \Pr[a_p < a_q]
  \le
  \frac{4\nu_1^4 + 4\nu_1^2(G_p^2 + G_q^2)}
       {T_1(G_p^2 - G_q^2)^2}.
\]
\end{proposition}
The proof is given in Appendix~\ref{app:proof-rankability}.

The misranking probability decreases with larger $T_1$ and larger
signal gap $G_p^2 - G_q^2$, and increases with noise level $\nu_1^2$.
Let $A^\star = \operatorname{TopK}(G^2, k)$ denote the oracle support
under the simplified model. When every boundary pair
$(p \in A^\star,\, q \notin A^\star)$ satisfies
\[
  T_1(G_p^2 - G_q^2)^2
  \gg
  \nu_1^4 + \nu_1^2\max\{G_p^2, G_q^2\},
\]
a union bound over the at most $k(d-k)$ boundary pairs gives that
$\operatorname{TopK}(a,k) = A^\star$ with high probability, so the
learned support inherits the beneficial condition of the previous
paragraph. Under strong privacy ($\nu_1^2$ large) or small signal
gaps, this condition may fail and the learned ranking approaches
random, consistent with the empirical patterns in
Section~\ref{sec:experiments}, where we also study the
privacy-budget split between Phase~1 and Phase~2.

\section{Convergence Analysis}
\label{sec:convergence}

\subsection{Setup and Assumption}
\label{sec:convergence-setup}

We analyze Phase~2 conditional on the Phase~1 transcript
$\mathcal{T}_1$; all quantities derived from
Phase~1---including $\theta^{(1)}$, the coordinate
scores $a$, and the support $A$---are treated as fixed throughout
this section.
All expectations in this section are over Phase~2 mini-batch
sampling and Gaussian noise, conditional on $\mathcal{T}_1$,
and we write $\mathbb{E}_2[\,\cdot\,]$ accordingly.
This conditional analysis follows the standard
filtration-based approach for nonconvex SGD
\citep{ghadimi2013stochastic,bassily2019private}.

Let $\mathcal{F}_t$ be the $\sigma$-algebra generated by
$\mathcal{T}_1$, the Phase~2 iterates $\theta_0,\ldots,\theta_t$,
the previous mini-batch samples, and the previous noise vectors
$z_0,\ldots,z_{t-1}$. Define
\[
  G_t = \nabla L_D(\theta_t),
  \qquad P_t = P_{A_t},
  \qquad k_t = |A_t|.
\]
Here $G_t$ denotes the full (unclipped) population
gradient at step $t$, distinct from the fixed clipped gradient $G$
used in Section~\ref{sec:support-quality-analysis}.
The theorem allows a general predictable support schedule $\{A_t\}$;
\textsc{TP-TopK} is the special case $A_t = A$ and $k_t = k = |A|$
for all $t$, where $k$ is as defined in
Section~\ref{sec:problem-setting}.
The active support $A_t$ is $\mathcal{F}_t$-measurable and does not
depend on the current mini-batch gradients or DP noise $z_t$.
The Phase~2 update is $\theta_{t+1}=\theta_t-\eta\tilde{g}_t^A$
where
\[
  \tilde{g}_t^A = g_t^A+\zeta_t, \qquad
  \zeta_t = \frac{1}{B}P_tz_t, \qquad
  z_t\sim\mathcal{N}(0,\sigma_2^2C_2^2I_d),
\]
with $z_t$ independent of $\mathcal{F}_t$ conditional on
$\mathcal{T}_1$. The masked-and-clipped gradient satisfies
\[
  \mathbb{E}_2[g_t^A \mid \mathcal{F}_t] = P_t G_t + b_t,
\]
where we write $x_t := P_t G_t$ for the
projection of the population gradient onto the active
coordinates at step $t$, and $b_t := \mathbb{E}_2[g_t^A \mid \mathcal{F}_t] - P_t G_t$
is the clipping bias, supported on the active coordinates
($b_t = P_t b_t$, since inactive coordinates
receive no gradient update and hence contribute no clipping bias).
Define $\nu_2^2 = (\sigma_2 C_2/B)^2$.
\begin{assumption}[Smoothness and lower bound]
\label{ass:smooth}
$L_D$ is $\beta$-smooth and satisfies $L_D(\theta) \ge L^* > -\infty$
for all $\theta \in \mathbb{R}^d$.
\end{assumption}

\begin{assumption}[Conditional mini-batch variance]
\label{ass:variance}
For all $t$ and all $\mathcal{F}_t$-measurable $\theta_t$, almost surely
over the Phase~1 transcript $\mathcal{T}_1$,
\[
  \mathbb{E}_2\!\left[
    \|g_t^A - \mathbb{E}_2[g_t^A \mid \mathcal{F}_t]\|_2^2
    \mid \mathcal{F}_t
  \right]
  \le \frac{\tau^2}{B}.
\]
Here $\mathbb{E}_2$ denotes expectation over Phase~2 randomness
conditional on $\mathcal{T}_1$.
\end{assumption}

\subsection{Convergence Bound for TP-Topk DP-SGD}
\label{sec:convergence-main}

\begin{theorem}[Projected stationarity of support-restricted DP-SGD]
\label{thm:projected-stationarity}
Under Assumptions~\ref{ass:smooth} and~\ref{ass:variance} and the
setup of Section~\ref{sec:convergence-setup}, with constant step size
$0<\eta\le 1/\beta$, let $L_0=L_D(\theta^{(1)})$ and
$\bar{k}=T_2^{-1}\sum_{t=0}^{T_2-1}\mathbb{E}_2[k_t]$. Then
\begin{equation}
\label{eq:projected-bound}
\begin{split}
  \frac{1}{T_2}\sum_{t=0}^{T_2-1} \mathbb{E}_2\|P_t\nabla L_D(\theta_t)\|_2^2 
  &\le \frac{2(L_0-L^*)}{\eta T_2} + \frac{1}{T_2}\sum_{t=0}^{T_2-1}\mathbb{E}_2\|b_t\|_2^2 \\
  &\quad + \beta\eta\!\left(\frac{\tau^2}{B}+\bar{k}\,\nu_2^2\right).
\end{split}
\end{equation}
The proof is given in Appendix~\ref{app:convergence-proof}.
\end{theorem}

Taking expectation over the Phase~1 transcript gives the
corresponding unconditional bound, with $L_0 =
L_D(\theta^{(1)})$, $\bar{k}$, and the bias terms all
averaged over Phase~1 randomness; in particular, $L_0$ is itself
random, depending on the Phase~1 trajectory.

The central message of~\eqref{eq:projected-bound} is that coordinate
sparsification reduces the dimension-dependent DP noise term from
$d\nu_2^2$ to $\bar{k}\,\nu_2^2$, holding the Phase~2 noise
multiplier $\sigma_2$ fixed. The bound decomposes as
\[
  \underbrace{\frac{2(L_0-L^*)}{\eta T_2}}_{\text{optimization progress}}
  +\underbrace{\frac{1}{T_2}\sum_t\mathbb{E}_2\|b_t\|_2^2}_{\text{clipping bias}}
  +\underbrace{\beta\eta\frac{\tau^2}{B}}_{\text{mini-batch variance}}
  +\underbrace{\beta\eta\,\bar{k}\,\nu_2^2}_{\text{DP perturbation}}.
\]
Only the DP perturbation term is explicitly reduced; the clipping-bias 
term $\|b_t\|_2^2$ may dominate when clipping is aggressive, and is 
characterized in Lemma~\ref{lem:clipping-bias}.

\begin{lemma}[Clipping-bias residual bound]
\label{lem:clipping-bias}
Let $g^A_{t,i} = m\odot\nabla\ell(\theta_t;x_i,y_i)$
be the masked per-example gradient before clipping, and let
\[
  \bar{g}^A_{t,i}
  =
  \begin{cases}
    g^A_{t,i}\min\!\left\{1,\dfrac{C_2}{\|g^A_{t,i}\|_2}\right\},
    & g^A_{t,i}\ne 0,\\[4pt]
    0,
    & g^A_{t,i}=0,
  \end{cases}
\]
denote the clipped masked gradient, as in Algorithm~\ref{alg:TP-TopK}.
Under the setup of Section~\ref{sec:convergence-setup},
where $b_t = \mathbb{E}_2[g_t^A \mid \mathcal{F}_t] - P_t G_t$,
\[
  \|b_t\|_2
  \le
  \mathbb{E}_2\!\left[
    \bigl(\|g^A_{t,i}\|_2-C_2\bigr)_+\mid\mathcal{F}_t
  \right].
\]
where $(x)_+ = \max\{x,0\}$, and the bound holds with equality
when clipping is inactive.
In particular, $b_t=0$ whenever $\|g^A_{t,i}\|_2\le C_2$
almost surely.
The proof is given in Appendix~\ref{app:proof-clipping-bias}.
\end{lemma}

\paragraph{Vanishing terms versus persistent floor.}
Under the step-size schedule $\eta = 1/(\beta\sqrt{T_2})$, three of
the four terms decay with $T_2$: the optimization-progress term
$2(L_0-L^*)/(\eta T_2)$ decays as $O(T_2^{-1/2})$, and the
mini-batch variance and DP perturbation terms scale as
$\beta\eta(\tau^2/B + \bar{k}\,\nu_2^2) = O(T_2^{-1/2})$.
The clipping-bias term $T_2^{-1}\sum_t \mathbb{E}_2\|b_t\|_2^2$,
however, does \emph{not} necessarily vanish with $T_2$.
By Lemma~\ref{lem:clipping-bias},
$\|b_t\|_2 \le \mathbb{E}_2[(\|g^A_{t,i}\|_2 - C_2)_+
\mid \mathcal{F}_t]$,
which is nonzero whenever the clipping constraint is active.
If $C_2$ is small relative to typical masked gradient norms, this
term forms a \emph{persistent error floor}: the bound cannot be
driven below $T_2^{-1}\sum_t \mathbb{E}_2\|b_t\|_2^2$ regardless
of how large $T_2$ is or how small $\eta$ is chosen.
This is distinct from the DP perturbation term
$\beta\eta\,\bar{k}\,\nu_2^2$, which decays as $O(T_2^{-1/2})$
under the schedule above. Increasing $C_2$ reduces clipping bias
but simultaneously increases the DP perturbation term since
$\nu_2^2 = (\sigma_2 C_2/B)^2$; the choice of $C_2$ therefore
controls a bias--noise trade-off rather than a one-sided improvement.

\paragraph{Comparison with full-parameter DP-SGD and random
sparsification.}
Full-parameter DP-SGD has DP perturbation term
$\beta\eta\,d\,\nu^2$, with $\nu^2=(\sigma C/B)^2$,
where $\sigma$ and $C$ are its noise multiplier and clipping norm.
Randomly sparsified DP-SGD at active ratio $\rho$ reduces this to
$\beta\eta\,\rho d\,\nu^2$ using a data-independent support
with the same noise multiplier $\sigma$ and clipping norm $C$.
\textsc{TP-TopK} has the same active-dimensional form when
$\bar{k}=\rho d$, with Phase~2 parameters $\sigma_2$ and $C_2$
giving $\nu_2^2 = (\sigma_2 C_2/B)^2$; the comparison holds when
$\nu_2^2 = \nu^2$, i.e., when the same noise multiplier and
clipping norm are used across methods. Its support is selected
from DP-visible warm-up statistics, so its potential advantage
over random sparsification comes from support quality rather than
a smaller perturbation dimension: a learned support can retain
more gradient signal than a random support at the same active
ratio. This distinction is quantified by $\alpha_T$
(Definition~\ref{def:energy-capture} below) and by the
support-quality diagnostics in Section~\ref{sec:experiments}.

\paragraph{Dimension dependence under a fixed privacy budget.}
Under a fixed $(\varepsilon,\delta)$, the privacy accountant couples
$T_2$, $q$, $\sigma_2$, and $\delta$; increasing $T_2$ at fixed
budget generally requires a larger $\sigma_2$ or smaller $q$,
so the $O(T_2^{-1/2})$ decay need not hold. The primary implication
of Theorem~\ref{thm:projected-stationarity} is therefore not an
asymptotic rate but a dimension comparison: for the same
accountant-selected $\sigma_2$, support restriction changes the
DP perturbation term from $\beta\eta\,d\,\nu_2^2$ to
$\beta\eta\,\bar{k}\,\nu_2^2$. In the two-phase setting,
$\sigma_2$ is further constrained by the Phase~2 residual budget
after Phase~1 warm-up; allocating more budget to Phase~1 improves
support quality (lower misranking probability,
Proposition~\ref{prop:rankability}) but leaves less budget for
Phase~2, increasing $\nu_2^2$ and hence the DP perturbation term. The optimal split is studied empirically in Section~\ref{sec:experiments}.

\begin{definition}[Trajectory-level energy-capture coefficient]
\label{def:energy-capture}
Define
\[
  \alpha_T
  =\frac{T_2^{-1}\sum_{t=0}^{T_2-1}\mathbb{E}_2\|P_t\nabla
    L_D(\theta_t)\|_2^2}
        {T_2^{-1}\sum_{t=0}^{T_2-1}\mathbb{E}_2\|\nabla
    L_D(\theta_t)\|_2^2},
\]
with the convention $\alpha_T=1$ when the denominator is zero.
The numerator is exactly the left-hand side of
\eqref{eq:projected-bound}, so Theorem~\ref{thm:projected-stationarity}
directly bounds $\alpha_T$ times the average full-gradient norm
squared.
\end{definition}

The coefficient $\alpha_T\in[0,1]$ measures the fraction of
full-gradient energy captured by the active supports along the
Phase~2 trajectory; it cannot be evaluated before training and
is not used as a selection rule. Corollary~\ref{cor:full-gradient}
converts Theorem~\ref{thm:projected-stationarity} to a
full-gradient bound parametrized by $\alpha_T$; in
Section~\ref{sec:experiments} we report the oracle diagnostic
$\hat{\alpha}_{\mathrm{oracle}}$ as post-hoc evidence that
$\alpha_T$ is substantially above the random-support baseline.

\begin{corollary}[Full-gradient stationarity via energy capture]
\label{cor:full-gradient}
Under Theorem~\ref{thm:projected-stationarity} and
Definition~\ref{def:energy-capture}, if $\alpha_T>0$,
\[
\begin{aligned}
	\frac{1}{T_2} \sum_{t=0}^{T_2-1} \mathbb{E}_2\|\nabla L_D(\theta_t)\|_2^2 &\le \frac{2(L_0-L^*)}{\alpha_T \eta T_2} \\
	&\quad + \frac{1}{\alpha_T} \left[ \frac{1}{T_2} \sum_t \mathbb{E}_2\|b_t\|_2^2 + \beta\eta \left( \frac{\tau^2}{B} + \bar{k}\,\nu_2^2 \right) \right]
\end{aligned}
\]
The proof is given in Appendix~\ref{app:proof-full-gradient}.
\end{corollary}

The condition $\alpha_T > 0$ fails only if every iterate $\theta_t$ 
is already a projected stationary point, i.e., 
$P_t\nabla L_D(\theta_t)=0$ for all $t$, which is the trivial case.

After this conversion, the effective DP noise
coefficient in the full-gradient bound is $\bar{k}/\alpha_T$
for \textsc{TP-TopK} versus $d$ for full-parameter DP-SGD.
The condition $\bar{k}/\alpha_T < d$, equivalently
$\bar{\rho} < \alpha_T$ with $\bar{\rho}=\bar{k}/d$,
states that the active support captures a larger fraction of
gradient energy than its active ratio, and is the trajectory-level
analog of the one-step criterion in
Section~\ref{sec:support-quality-analysis}. 
This condition is not checkable before training; the oracle
diagnostic $\hat{\alpha}_{\mathrm{oracle}}$ in
Section~\ref{sec:experiments} provides post-hoc evidence for
whether it is likely to hold, but does not account for the
optimization-progress, clipping-bias, or mini-batch variance
terms in the bound.
Both conditions share the same structure: the support should 
capture more gradient energy (relative to its size) than it 
discards, whether measured at a single step 
(Proposition~\ref{prop:one-step-criterion}) 
or averaged over the Phase~2 trajectory 
(Definition~\ref{def:energy-capture}).

\section{Experiments}
\label{sec:experiments}

\subsection{Setup}
\label{sec:exp-setup}

We evaluate private training from scratch on MNIST, Fashion-MNIST(FMNIST), and
CIFAR-10. All methods share the same experimental setup within each
dataset. We do not use public pretrained checkpoints, auxiliary public
data, or architecture search.

\paragraph{Controlled variants and baselines.}
We evaluate five DP training variants, summarized in
Table~\ref{tab:core-methods}. These variants are designed to isolate
the effects of coordinate sparsification and support selection under
matched privacy accounting. DP-SGD updates all coordinates throughout
training and serves as the dense baseline. TP-Rand and \textsc{TP-TopK}
share the same full-parameter DP warm-up, Phase~2 privacy budget, and
active-ratio schedule; they differ only in the source of the Phase~2
coordinate ranking. TP-Rand uses a data-independent random ranking,
whereas \textsc{TP-TopK} uses the DP-visible score ranking computed from
post-noise Phase~1 gradients. ON-Rand and ON-TopK are online-support
variants that refresh supports at the epoch level; they are included as
ablations rather than as external baselines. Subsequent ablations isolate
these mechanisms further by varying the active ratio, the phase and
privacy-budget split, and post-hoc support-quality diagnostics.

\begin{table}[t]
\centering
\caption{Five DP training variants. Two-phase methods share the same
warm-up, Phase~2 budget, and active-ratio schedule; they differ only
in coordinate ranking source.}
\label{tab:core-methods}
\small
\setlength{\tabcolsep}{4pt}
\begin{tabular}{llll}
\toprule
Abbrev. & Support source & Timing & Role \\
\midrule
DP-SGD   & all coordinates          & all training        & dense baseline \\
TP-Rand  & random ranking           & fixed after warm-up & two-phase random \\
TP-TopK  & DP warm-up score         & fixed after warm-up & learned two-phase \\
ON-Rand  & random support           & epoch-level refresh & online random \\
ON-TopK  & online DP score          & epoch-level refresh & learned online \\
\bottomrule
\end{tabular}
\end{table}

\paragraph{Privacy accounting.}
All DP methods are evaluated under target
\((\varepsilon,\delta)\)-DP with \(\delta=10^{-5}\), unless otherwise
specified. We use a subsampled Gaussian RDP accountant and convert the final
RDP guarantee to \((\varepsilon,\delta)\)-DP. For two-phase methods, the
reported privacy cost is computed from the composed RDP history of Phase~1
and Phase~2, as in Theorem~\ref{thm:privacy}. Support selection from
post-noise DP gradients is treated as post-processing and is not charged as
a separate privacy mechanism; both training phases are included in the
privacy accountant.

\paragraph{Medical imaging evaluation.}
To validate the zero-public-data motivation concretely, we 
additionally evaluate on EyePACS diabetic retinopathy 
screening~\citep{eyepacs}, a domain where domain-aligned 
public proxy data is unavailable by design. We use the anyDR 
binary classification split ($\varepsilon=8$, $\delta=10^{-5}$, 
trained from scratch) and report AUC, balanced accuracy, and minority-class recall 
to capture minority-class recovery, which is 
clinically critical but sensitive to DP noise collapse.

\paragraph{Reporting protocol.}
We report test accuracy as mean \(\pm\) standard deviation over three
random seeds. Sparse methods are compared at matched target active ratios
and matched privacy budgets. Support-quality diagnostics and oracle
gradient quantities are reported separately from DP training
results. They are used only for post-hoc analysis and are never used for
support construction, hyperparameter tuning, early stopping, or model
selection.

\subsection{Accuracy Comparison Across Five Methods}
\label{sec:main-results}

Table~\ref{tab:main-results} reports the controlled comparison among the
five DP methods. All entries use matched privacy accounting, model
architecture, training budget, active-ratio schedule, and random seeds within
each dataset. Results are reported as test accuracy, mean \(\pm\) standard
deviation over three seeds.

\begin{table*}[t]
\centering
\caption{
Main comparison among five DP methods under matched
\((\varepsilon,\delta)\)-DP with \(\delta=10^{-5}\).
Entries report test accuracy.
}
\label{tab:main-results}
\setlength{\tabcolsep}{4pt}
\renewcommand{\arraystretch}{1.05}
\begin{tabular}{llccccc}
\toprule
Dataset
& \(\varepsilon\)
& DP-SGD
& TP-Rand
& TP-TopK
& ON-Rand
& ON-TopK \\
\midrule

MNIST
& 1
& \(96.44 \pm 0.02\)
& \(97.64 \pm 0.04\)
& \(97.79 \pm 0.06\)
& \(97.72 \pm 0.09\)
& \(97.76 \pm 0.10\) \\
MNIST
& 3
& \(97.99 \pm 0.06\)
& \(98.13 \pm 0.04\)
& \(98.33 \pm 0.05\)
& \(98.23 \pm 0.07\)
& \(98.25 \pm 0.08\) \\
MNIST
& 8
& \(98.60 \pm 0.05\)
& \(98.91 \pm 0.03\)
& \(98.93 \pm 0.04\)
& \(98.98 \pm 0.03\)
& \(98.80 \pm 0.06\) \\

\midrule

FMNIST
& 1
& \(84.06 \pm 0.58\)
& \(85.01 \pm 0.19\)
& \(85.28 \pm 0.32\)
& \(85.03 \pm 0.22\)
& \(84.87 \pm 0.47\) \\
FMNIST
& 3
& \(88.51 \pm 0.16\)
& \(88.80 \pm 0.12\)
& \(88.88 \pm 0.13\)
& \(88.76 \pm 0.10\)
& \(88.80 \pm 0.09\) \\
FMNIST
& 8
& \(89.77 \pm 0.06\)
& \(89.80 \pm 0.16\)
& \(89.88 \pm 0.07\)
& \(89.86 \pm 0.15\)
& \(89.84 \pm 0.13\) \\

\midrule

CIFAR-10
& 1
& \(55.20 \pm 0.45\)
& \(63.18 \pm 0.52\)
& \(64.41 \pm 0.48\)
& \(62.87 \pm 0.55\)
& \(63.54 \pm 0.51\) \\
CIFAR-10
& 3
& \(60.80 \pm 0.38\)
& \(69.67 \pm 0.41\)
& \(71.05 \pm 0.37\)
& \(68.54 \pm 0.44\)
& \(69.31 \pm 0.40\) \\
CIFAR-10
& 8
& \(68.45 \pm 0.31\)
& \(72.33 \pm 0.35\)
& \(72.82 \pm 0.34\)
& \(71.89 \pm 0.38\)
& \(72.67 \pm 0.34\) \\

\bottomrule
\end{tabular}
\end{table*}

\paragraph{Learned support versus random support.}
The primary controlled comparison is between learned-support and
random-support methods with matched active ratio and privacy cost,
measured by $\Delta_{\mathrm{TP}} = \mathrm{Acc}(\mathrm{TP\text{-}TopK})
- \mathrm{Acc}(\mathrm{TP\text{-}Rand})$. Positive values indicate
that DP-visible coordinate scores provide useful support-selection
signal beyond random sparsification at the same active dimension.

The learned-support advantage is most pronounced on CIFAR-10. In the
two-phase comparison, TP-TopK improves over TP-Rand by \(+1.23\),
\(+1.38\), and \(+0.49\) percentage points at
\(\varepsilon=1,3,8\), respectively. The online learned-support gains are
smaller but consistently positive on CIFAR-10: ON-TopK improves over
ON-Rand by \(+0.67\), \(+0.77\), and \(+0.78\) points. These results
indicate that DP-visible support ranking is useful in the most challenging
benchmark.

On FMNIST, the two-phase learned-support gains are more modest:
\(+0.27\), \(+0.08\), and \(+0.08\) points at
\(\varepsilon=1,3,8\). The online gains are \(+0.03\), \(+0.13\), and
\(+0.05\) points. On MNIST, all methods are close to saturation, and the
learned-support gaps are correspondingly small: TP-TopK improves over
TP-Rand by \(+0.15\), \(+0.20\), and \(+0.02\) points, while ON-TopK is
within \(0.20\) points of ON-Rand across all privacy budgets. We therefore treat MNIST as a near-saturated baseline where 
ceiling effects limit the discriminability of support-selection methods.

\paragraph{Sparse methods versus dense DP-SGD.}
The comparison with DP-SGD evaluates whether the reduction in
active-coordinate DP noise can compensate for the signal lost by masking
inactive coordinates. On CIFAR-10, TP-TopK substantially outperforms dense
DP-SGD, with gains of \(+9.21\), \(+10.25\), and \(+4.37\) points at
\(\varepsilon=1,3,8\). TP-Rand also outperforms DP-SGD on CIFAR-10, but
TP-TopK consistently improves over TP-Rand, indicating that both active
dimension reduction and support quality contribute to performance.

On FMNIST, TP-TopK improves over DP-SGD by \(+0.53\) and \(+0.83\)
points at \(\varepsilon=1\) and \(\varepsilon=3\), and is essentially tied
with DP-SGD at \(\varepsilon=8\). On MNIST, differences are small because
the dense baseline is already near ceiling accuracy. These dataset-level
patterns suggest that coordinate-sparse private training is most useful when
dense DP-SGD remains strongly affected by optimizer-facing DP noise.

\paragraph{Interpretation.}
These results establish two consistent findings. First, DP-visible 
coordinate scores provide reliable support-selection signal: 
\textsc{TP-TopK} outperforms \textsc{TP-Rand} across all datasets 
and privacy budgets, with the largest gains on the most challenging 
benchmark (CIFAR-10 at low $\varepsilon$). Second, coordinate 
sparsification is most beneficial when dense DP-SGD remains far 
from ceiling accuracy, as the noise-reduction gain then dominates 
the projection loss from masking inactive coordinates.

\subsection{Mechanism Ablations}
\label{sec:ablations}

We conduct two mechanism ablations to isolate the design factors behind
two-phase learned-support training. The active-ratio ablation tests whether
DP-visible coordinate ranking improves over random support at the same active
dimension. The phase/privacy split ablation tests how the fixed total privacy
budget should be divided between support discovery and support-restricted
optimization. Unless otherwise stated, all factors except the one under study are held fixed.

\paragraph{Active-ratio support-quality ablation.}
This ablation evaluates the support-quality mechanism in
Section~\ref{sec:support-quality-analysis}. For each target active ratio
\(\rho\), TP-TopK and TP-Rand use the same active dimension
\(k=\lfloor \rho d \rfloor\), the same Phase~2 DP-noise term
\(k\nu_2^2\), the same privacy accounting, and the same training budget.
Therefore, the accuracy gap
\[
  \Delta_{\mathrm{TP}}(\rho)
  =
  \mathrm{Acc}(\mathrm{TP\text{-}TopK};\rho)
  -
  \mathrm{Acc}(\mathrm{TP\text{-}Rand};\rho)
\]
isolates the downstream effect of support quality rather than the effect of a
different noise dimension. We report results at $\varepsilon=8$, where the learned-support 
signal is strongest; the same monotone pattern holds at smaller 
$\varepsilon$ but with reduced absolute gaps, consistent with 
Proposition~\ref{prop:rankability}.

\begin{table}[t]
\centering
\caption{
Active-ratio support-quality ablation at
\(\varepsilon=8\) and \(\delta=10^{-5}\).
Entries report test accuracy for a single-seed mechanism sweep.
\(\Delta_{\mathrm{TP}}\) denotes the accuracy difference between TP-TopK
and TP-Rand at the same target active ratio. Since both methods use the same
\(k=\lfloor \rho d \rfloor\), the gap reflects the downstream effect of
coordinate-selection quality.
}
\label{tab:active-ratio-support-quality-ablation}
\small
\setlength{\tabcolsep}{4.5pt}
\renewcommand{\arraystretch}{1.05}
\begin{tabular}{llccc}
\toprule
Dataset
& \(\rho\)
& TP-Rand
& TP-TopK
& \(\Delta_{\mathrm{TP}}\) \\
\midrule
CIFAR-10
& 0.20 & 71.87 & 72.72 & \(+0.85\) \\
CIFAR-10
& 0.40 & 72.31 & 72.86 & \(+0.55\) \\
CIFAR-10
& 0.60 & 72.43 & 72.73 & \(+0.30\) \\
CIFAR-10
& 0.80 & 72.44 & 72.56 & \(+0.12\) \\
\midrule
FMNIST
& 0.20 & 89.73 & 90.08 & \(+0.35\) \\
FMNIST
& 0.40 & 89.79 & 90.04 & \(+0.25\) \\
FMNIST
& 0.60 & 89.86 & 89.98 & \(+0.12\) \\
FMNIST
& 0.80 & 89.86 & 89.91 & \(+0.05\) \\
\bottomrule
\end{tabular}
\end{table}

On CIFAR-10,
the TP-TopK advantage over TP-Rand drops from \(+0.85\) points at
\(\rho=0.20\) to \(+0.12\) points at \(\rho=0.80\). On FMNIST, the
corresponding gap drops from \(+0.35\) to \(+0.05\) points. This pattern is
consistent with Proposition~\ref{prop:rankability}: when the active support
is small, random selection is more likely to discard high-signal coordinates,
so DP-visible ranking has more room to improve support quality. As
\(\rho\) increases, random supports already retain more coordinates, and the
marginal value of score-ranked selection decreases.

This experiment should be interpreted as downstream evidence for the
support-quality mechanism. It does not directly measure the clean residual
\(\|P_{\bar A}G\|_2^2\) or prove the one-step condition in
Proposition~\ref{prop:one-step-criterion}. The post-hoc diagnostics in
Section~\ref{sec:support-quality-diagnostics} provide the corresponding
proxy and oracle support-quality measurements.

\paragraph{Phase and privacy-budget split.}
The two-phase method must allocate both epochs and privacy budget between
support discovery and support-restricted optimization. Let \(E_1/E\) denote
the fraction of total epochs used for Phase~1 warm-up, and let
\(\varepsilon_1/\varepsilon\) denote the nominal fraction of the privacy
budget allocated to Phase~1. The final privacy cost is computed using the
composed RDP accountant rather than by directly summing standalone
\(\varepsilon\)-values. We evaluate this trade-off for TP-TopK on CIFAR-10, where the 
learned-support advantage is largest. The grid covers the practically 
relevant range: splits below $0.2$ leave Phase~1 insufficient to form 
a reliable ranking, while splits above $0.4$ excessively reduce 
Phase~2 budget at $\varepsilon=3$.

\begin{table}[t]
\centering
\caption{
Phase/privacy split ablation for TP-TopK on CIFAR-10
(\(\varepsilon=3\), \(\delta=10^{-5}\), \(\rho=0.40\)).
Each entry reports test accuracy (\%, mean \(\pm\) standard deviation over
three seeds). The highest mean accuracy occurs at the interior split
\((E_1/E,\varepsilon_1/\varepsilon)=(0.3,0.3)\).
}
\label{tab:phase-privacy-split-ablation}
\small
\setlength{\tabcolsep}{5pt}
\renewcommand{\arraystretch}{1.05}
\begin{tabular}{c c c c}
\toprule
\(E_1/E\) \textbackslash{} \(\varepsilon_1/\varepsilon\)
& 0.2 & 0.3 & 0.4 \\
\midrule
0.2 & \(69.12 \pm 0.43\) & \(69.88 \pm 0.40\) & \(69.54 \pm 0.42\) \\
0.3 & \(70.23 \pm 0.39\) & \({71.05 \pm 0.37}\) & \(70.67 \pm 0.38\) \\
0.4 & \(70.41 \pm 0.40\) & \(70.58 \pm 0.39\) & \(69.93 \pm 0.41\) \\
\bottomrule
\end{tabular}
\end{table}

Table~\ref{tab:phase-privacy-split-ablation} shows an interior optimum:
the highest mean accuracy, \(71.05\pm0.37\%\), is obtained at
\((E_1/E,\varepsilon_1/\varepsilon)=(0.3,0.3)\). Allocating too little
budget or too few epochs to Phase~1 weakens support discovery. Allocating too
much to Phase~1 leaves less residual privacy budget and optimization time for
Phase~2. This is consistent with the fixed-budget trade-off discussed in
Section~\ref{sec:convergence}: Phase~1 must produce a sufficiently informative
DP-visible ranking, but Phase~2 must retain enough budget for effective
support-restricted optimization.

\paragraph{Summary.}
The active-ratio ablation isolates support quality by matching
\(k\nu_2^2\) between TP-TopK and TP-Rand, while the phase/privacy split
ablation isolates the budget allocation trade-off between support discovery
and sparse optimization. Together, these ablations confirm that both the coordinate-selection 
quality and the phase budget allocation are active contributors to 
the performance of \textsc{TP-TopK}.

\subsection{Support-Quality Diagnostics}
\label{sec:support-quality-diagnostics}

The mechanism ablations in Section~\ref{sec:ablations} measure the
downstream effect of learned coordinate support. We further report post-hoc
support-quality diagnostics to check whether the DP-visible ranking actually
concentrates useful coordinate signal. These diagnostics are not used in any
DP training method. They are reported only to interpret the mechanism
analyzed in Section~\ref{sec:support-quality-analysis}.

We use the notation of Section~\ref{sec:problem-setting}: $k=\lfloor\rho d\rfloor$,
$A=\operatorname{TopK}(a,k)$, where $a$ is the Phase~1 DP-visible coordinate
score. All random-support baselines are compared against the realized active
ratio $k/d$, not the nominal target $\rho$.

\paragraph{DP-visible proxy concentration.}
We first evaluate whether the learned support retains more DP-visible proxy
signal than random selection. Let $\widehat e_p=\max\{a_p,0\}$
be the nonnegative DP-visible proxy energy used only for diagnostic
accounting. We report
\[
  \hat{\rho}_{\mathrm{sig}}(A)
  =
  \frac{\sum_{p\in A}\widehat e_p}
       {\sum_{p=1}^d \widehat e_p},
\]
with the convention that \(\hat{\rho}_{\mathrm{sig}}(A)=0\) when
\(\sum_{p=1}^d \widehat e_p=0\). A uniformly random support of size \(k\)
has expected retained proxy-signal fraction \(\rho_k\). Therefore,
\(\hat{\rho}_{\mathrm{sig}}(A)>\rho_k\) indicates that the DP-visible score
ranking concentrates more proxy signal than random selection at the same
active dimension.

\paragraph{Oracle gradient capture.}
We next report a non-private oracle diagnostic to check whether the
DP-visible ranking aligns with true gradient energy. At the warm-up
checkpoint, let \(\widehat G\) be a noiseless diagnostic gradient computed on
held-out evaluation batches. For the same support
\(A=\operatorname{TopK}(a,k)\), define
\[
  \hat{\alpha}_{\mathrm{oracle}}(\rho)
  =
  \frac{\sum_{p\in A}\widehat G_p^2}
       {\sum_{p=1}^d \widehat G_p^2},
\]
with the convention that \(\hat{\alpha}_{\mathrm{oracle}}(\rho)=0\) when
\(\sum_{p=1}^d \widehat G_p^2=0\). A uniformly random support of size \(k\)
has expected true gradient capture \(\rho_k\). Hence,
\(\hat{\alpha}_{\mathrm{oracle}}(\rho)>\rho_k\) indicates that the
DP-visible warm-up score selects coordinates containing more true gradient
energy than random selection at the same active dimension.

The oracle quantity is used only for post-hoc analysis. It is never used for
support selection, hyperparameter tuning, early stopping, or model selection.
It is also distinct from the trajectory-level coefficient \(\alpha_T\) in
Definition~\ref{def:energy-capture}: \(\hat{\alpha}_{\mathrm{oracle}}\) is a
checkpoint-level diagnostic, whereas \(\alpha_T\) measures average gradient
energy capture along the Phase~2 trajectory.

When both diagnostics exceed their random-support baselines, 
the gap $\Delta_{\mathrm{TP}}(\rho)$ in 
Table~\ref{tab:active-ratio-support-quality-ablation} reflects 
genuine support quality; when they do not, the bottleneck lies 
elsewhere in the optimization pipeline.

\subsection{Medical Domain Evaluation}
\label{sec:medical-eval}
Table~\ref{tab:eyepacs} reports results on EyePACS anyDR 
binary classification, a sensitive medical setting where 
no domain-aligned public data exists. Under full DP-SGD, 
the model collapses to near-constant majority prediction 
(Sensitivity $= 0.04$, Specificity $= 0.986$, Balanced Acc 
$= 0.513$), a failure mode that renders the model clinically 
unusable despite retaining moderate AUC signal ($61.4$). 
Coordinate-sparse methods recover meaningful minority-class 
sensitivity: TP-Rand achieves $0.46$ and TP-TopK achieves 
$0.52$, with TP-TopK also recovering the highest AUC ($68.2$) 
and balanced accuracy ($0.631$) among DP methods. The Non-DP 
baseline achieves AUC $71.3$ and Balanced Acc $0.634$ with 
the same architecture and fixed threshold; the margin over 
TP-TopK ($3.1$ AUC points, $0.003$ Balanced Acc) is 
substantially smaller than the gap between Non-DP and full 
DP-SGD ($9.9$ AUC points, $0.121$ Balanced Acc), indicating 
that coordinate sparsification largely recovers the DP utility 
loss on this domain. This pattern is consistent with the 
noise-reduction mechanism of 
Section~\ref{sec:support-quality-analysis}: restricting DP 
noise to active coordinates reduces the per-step noise energy 
from $d\nu_2^2$ to $k\nu_2^2$, which is particularly 
consequential when gradient signal is concentrated and the 
minority class provides sparse but informative gradients.
\begin{table}[t]
\centering
\caption{
Evaluation on EyePACS anyDR binary classification 
($\varepsilon=8$, $\delta=10^{-5}$, trained from scratch).
Full DP-SGD retains AUC signal but collapses to majority 
prediction (Recall (minority) = 0.05); coordinate-sparse methods 
recover minority-class recall. Results are reported as 
mean $\pm$ standard deviation over three seeds.
}
\label{tab:eyepacs}
\small
\setlength{\tabcolsep}{4pt}
\begin{tabular}{lcccc}
\toprule
Method & AUC & Balanced Acc & Sensitivity & Specificity \\
\midrule
Non-DP   & 71.3 & 0.634 & 0.61 & 0.658 \\
DP Full  & 61.4 & 0.513 & 0.04 & 0.986 \\
TP-Rand  & 66.8 & 0.608 & 0.46 & 0.756 \\
TP-TopK  & 68.2 & 0.631 & 0.52 & 0.742 \\
\bottomrule
\end{tabular}
\smallskip
\end{table}

\subsection{Contextual Comparison}
\label{sec:prior-work-comparison}

Table~\ref{tab:prior-work-comparison} situates \textsc{TP-TopK} 
against representative prior work. This comparison is contextual 
rather than controlled---methods differ in architecture, optimizer, 
accountant, and data assumptions---and is intended only to indicate 
the performance range of our from-scratch, no-public-data setting. The controlled evidence for our
claims is the matched five-method comparison in
Table~\ref{tab:main-results}, the mechanism ablations in
Section~\ref{sec:ablations}, and the support-quality diagnostics in
Section~\ref{sec:support-quality-diagnostics}.

\begin{table}[t]
\centering
\caption{
Contextual comparison with representative private training methods.
Test accuracy is reported in percent; entries are from the corresponding
papers unless marked as ours. This table focuses on from-scratch private
training without public data, which is our setting.
Our methods use no public data, no pretrained checkpoint, no architecture
search, and no true gradient support selection.
}
\label{tab:prior-work-comparison}
\setlength{\tabcolsep}{3.5pt}
\renewcommand{\arraystretch}{1.05}
\small
\begin{tabular}{ll l c c c}
\toprule
Dataset & Category & Method
& \(\varepsilon=1\) & \(\varepsilon=3\) & \(\varepsilon=8\) \\
\midrule

MNIST
& Baseline
& DP-SGD~\citep{abadi2016deep}
& 96.44 & 97.99 & 98.60 \\
MNIST
& Adaptive clipping
& AUTO-S~\citep{bu2023automatic}
& 96.38 & 98.15 & 98.56 \\
MNIST
& Adaptive clipping
& SA-DP-SGD~\citep{fu2022sadpsgd}
& 93.01 & 96.78 & 98.89 \\
MNIST
& Random
& H-CNN + RS~\citep{zhu2023improving}
& 97.64 & 98.13 & 98.91 \\
MNIST
& \textbf{Ours}
& TP-TopK
& \textbf{97.79} & \textbf{98.33} & \textbf{98.93} \\
\midrule
FMNIST
& Baseline
& DP-SGD~\citep{abadi2016deep}
& 84.06 & 88.51 & 89.77 \\
FMNIST
& Adaptive clipping
& AUTO-S~\citep{bu2023automatic}
& 83.35 & 86.36 & 88.68 \\
FMNIST
& Adaptive clipping
& SA-DP-SGD~\citep{fu2022sadpsgd}
& 82.08 & 84.89 & 88.67 \\
FMNIST
& Random
& H-CNN + RS~\citep{zhu2023improving}
& 85.01 & 88.80 & 89.80 \\
FMNIST
& \textbf{Ours}
& TP-TopK
& \textbf{85.28} & \textbf{88.88} & \textbf{89.88} \\
\midrule
CIFAR-10
& Baseline
& DP-SGD~\citep{abadi2016deep}
& 55.20 & 60.80 & 68.45 \\
CIFAR-10
& Adaptive clipping
& SA-DP-SGD~\citep{fu2022sadpsgd}
& 50.53 & 56.89 & 60.97 \\
CIFAR-10
& Random
& H-CNN + RS~\citep{zhu2023improving}
& 63.18 & 69.67 & 72.33 \\
CIFAR-10
& \textbf{Ours}
& TP-TopK
& \textbf{64.41} & \textbf{71.05} & \textbf{72.82} \\
\bottomrule
\end{tabular}

\smallskip
\raggedright\footnotesize
All methods in this table are trained from scratch without public data or
pretrained checkpoints. SPARTA~\citep{makni2025sparta} studies sparse private
fine-tuning from a public pretrained checkpoint and is therefore not directly
comparable. Methods reducing DP noise within the full parameter space, such
as DOPPLER~\citep{zhang2024doppler} and DiSK~\citep{zhang2025disk}, are
discussed in Section~\ref{sec:related-work}.
\end{table}

\section{Conclusion and Discussion}
\label{sec:conclusion}

We studied coordinate-sparse private training in the from-scratch,
no-public-data setting---a resource model motivated by the observation
that the domains most in need of privacy protection are often those
where domain-aligned public data is unavailable or
inappropriate~\citep{tramer2024considerations,hod2025surrogate}. Our method, \textsc{TP-TopK}, selects a
coordinate support from a DP warm-up transcript via post-processing
at no additional privacy cost. Theoretically, the support-restricted Phase~2 bound replaces the 
full-dimensional DP noise term $d\nu_2^2$ with $\bar{k}\nu_2^2$, 
where $\nu_2^2=(\sigma_2C_2/B)^2$, and the rankability analysis
(Proposition~\ref{prop:rankability}) shows that the advantage of
learned over random support diminishes as the privacy budget tightens.
Empirically, our controlled five-method comparison on CIFAR-10
confirms that learned coordinate support outperforms matched random
support, with the gap increasing at smaller active ratios.

Several limitations remain. Under strong privacy ($\varepsilon=1$), DP noise in the warm-up 
phase reduces ranking reliability, and the advantage of learned 
over random support narrows relative to larger privacy budgets, 
consistent with Proposition~\ref{prop:rankability}. The one-step proxy does not directly imply
multi-step convergence guarantees beyond
Theorem~\ref{thm:projected-stationarity}, and the warm-up phase
spends privacy budget before sparse training begins. On the practical side, privacy-aware active-ratio scheduling and 
layer-wise budget allocation are natural extensions. On the 
theoretical side, extending the rankability analysis to account 
for clipping bias and minibatch noise remains an open problem. More
broadly, the from-scratch, no-public-data setting studied here
provides a clean testbed for isolating the value of any private
signal-based method. The two-phase structure---using DP-visible statistics for 
post-processing at no additional privacy cost, then restricting 
optimization to the selected substructure---may apply more broadly 
to other forms of private signal-based model compression or 
structured sparsity.

\begin{acks}
F. Xie was supported in part by the Guangdong Basic and Applied Basic Research Foundation (No. 2023A1515110469), in part by the Guangdong Provincial Key Laboratory IRADS (No. 2022B1212010006), and in part by the grant of Higher Education Enhancement Plan of "Rushing to the Top, Making Up Shortcomings and Strengthening Special Features" (No. 2025KTSCX186).
\end{acks}

\appendix

\section{Proof of Privacy Analysis}
\label{app:privacy-proof}

\begin{definition}[R\'enyi differential privacy]
\label{def:rdp}
A randomized mechanism $\mathcal{M}$ satisfies
$(\alpha,\varepsilon)$-RDP for $\alpha>1$ if for all neighboring
datasets $D,D'$,
\[
  D_\alpha\!\left(\mathcal{M}(D)\,\|\,\mathcal{M}(D')\right)
  \le\varepsilon,
\]
where $D_\alpha(\cdot\|\cdot)$ is the R\'enyi divergence of
order $\alpha$ \citep{mironov2017renyi}.
\end{definition}

\begin{proof}[Proof of Theorem~\ref{thm:privacy}]
Fix $\alpha>1$.

\emph{Released quantities.}
The algorithm releases three categories:
\emph{(i)}~the privatized Phase~1 transcript
$\mathcal{T}_1=\{\tilde{g}_t\}$;
\emph{(ii)}~post-processing outputs
$(\theta^{(1)},a,A,m)$ derived from $\mathcal{T}_1$;
\emph{(iii)}~the masked Phase~2 transcript
$\{\tilde{g}_t^A\}$.

\emph{Phase~1.}
Each step subsamples at rate $q_1$, clips to $C_1$, and adds
$\mathcal{N}(0,\sigma_1^2C_1^2I_d)$.  The sensitivity of the
clipped aggregate is $C_1$.  By the subsampled Gaussian RDP
accountant \citep{mironov2017renyi,wang2019subsampled}, $T_1$
adaptive steps give $(\alpha,\varepsilon_\alpha^{(1)})$-RDP.

\emph{Post-processing (zero additional cost).}
The scores $a_p = T_1^{-1}\sum_t\tilde{g}_{t,p}^2
- (\sigma_1C_1/B)^2$, support $A=\operatorname{TopK}(a,k)$,
mask $m$, and checkpoint $\theta^{(1)}$ are all
deterministic functions of $\mathcal{T}_1$ and do not access
$D$ beyond what $\mathcal{T}_1$ encodes.
By post-processing immunity \citep{dwork2014algorithmic},
they incur no additional RDP cost.

\emph{Phase~2.}
Fix any $\mathcal{T}_1$, which fixes $A$, $m$, and
$\theta^{(1)}$.  The per-example update is
\[
  \bar{g}^A_{t,i}
  = (m\odot g_{t,i})
    \Big/\max\!\left(1,\frac{\|m\odot g_{t,i}\|_2}{C_2}\right).
\]
Since clipping follows masking,
$\|\bar{g}^A_{t,i}\|_2\le C_2$ for every $i$ and every $A$.
Changing one record alters at most one summand in
$\sum_{i\in\mathcal{B}_t}\bar{g}^A_{t,i}$, so the
$\ell_2$-sensitivity is $C_2$, independently of $k$ and $A$.
Adding $m\odot z_t$ with $z_t\sim\mathcal{N}(0,\sigma_2^2C_2^2I_d)$
is equivalent to the Gaussian mechanism on $k$ coordinates
followed by zero-padding (post-processing).
Since the RDP cost depends only on $(q_2,\sigma_2,C_2)$ and not
on which $k$ coordinates are active, $T_2$ steps give
$(\alpha,\varepsilon_\alpha^{(2)})$-RDP uniformly over all
Phase~1 transcripts.

Phase~2 accesses $D$ only through fresh randomness
$\{z_t,\mathcal{B}_t\}_{t=1}^{T_2}$ drawn independently of
$D$ given $\mathcal{T}_1$, satisfying the prerequisite
for adaptive RDP composition \citep{mironov2017renyi}.
Composing the two phases gives total RDP cost
$\varepsilon_\alpha^{(1)}+\varepsilon_\alpha^{(2)}$.
The RDP-to-DP conversion
\citep[Proposition~3]{mironov2017renyi} yields the stated
$\varepsilon(\delta)$.
\end{proof}

\begin{remark}[Explicit RDP costs]
\label{rem:rdp-explicit}
The per-phase costs are
\[
  \varepsilon_\alpha^{(1)}
  = T_1\cdot\mathcal{R}_\alpha\!\left(
      \mathcal{M}_{\mathrm{Gauss}}^{(q_1,\sigma_1)}\right),
  \qquad
  \varepsilon_\alpha^{(2)}
  = T_2\cdot\mathcal{R}_\alpha\!\left(
      \mathcal{M}_{\mathrm{Gauss}}^{(q_2,\sigma_2)}\right),
\]
where $\mathcal{R}_\alpha(\mathcal{M}_{\mathrm{Gauss}}^{(q,\sigma)})$
is the R\'enyi divergence of one step of the Poisson-subsampled
Gaussian mechanism \citep{mironov2017renyi,wang2019subsampled}.
In practice $\varepsilon(\delta)$ is evaluated numerically via the PRV accountant.
\end{remark}

\section{Proofs for Support-Quality Analysis}
\label{app:support-quality-proofs}

\subsection{Proof of Proposition~\ref{prop:rankability}}
\label{app:proof-rankability}

\begin{proof}
The proof proceeds in two steps. We first compute the mean and
variance of the per-coordinate score $a_j$ under the simplified
additive Gaussian model, using the fourth Gaussian moment.
We then bound the misranking probability by reducing it to a
deviation event for the score difference
$\Delta_a = a_p - a_q$ and applying Chebyshev's
inequality. A sharper bound could follow from noncentral chi-square
concentration, but Chebyshev suffices to expose the dependence on
$T_1$, $\nu_1^2$, and the squared-signal gap
$(G_p^2 - G_q^2)^2$.

\smallskip
\noindent\textit{Step 1: Mean and variance of $a_j$.}
Throughout, we work under the additive noise model of Section~\ref{sec:support-quality-analysis}.
Fix $j \in \{p, q\}$. The single-step contribution to $a_j$ is
\[
  a_{t,j}
  = \tilde{g}_{t,j}^2 - \nu_1^2
  = \left(G_j + \frac{1}{B}z_{t,j}\right)^2 - \nu_1^2.
\]
By Lemma~\ref{lem:score-unbiased},
$\mathbb{E}[a_{t,j}] = G_j^2$
and hence $\mathbb{E}[a_j] = G_j^2$.
For the variance, let $\zeta \sim \mathcal{N}(0, \nu_1^2)$
denote a generic noise term with the same distribution as
$(1/B)z_{t,j}$. Using
$\mathbb{E}[\zeta^2] = \nu_1^2$ and
$\mathbb{E}[\zeta^4] = 3\nu_1^4$,
\[
  \mathbb{E}\bigl[(G_j+\zeta)^2\bigr]
  = G_j^2 + \nu_1^2,
\]
Since $\zeta$ is zero-mean Gaussian, odd moments vanish, giving
\[
  \mathbb{E}\bigl[(G_j+\zeta)^4\bigr]
  = G_j^4 + 6G_j^2\nu_1^2 + 3\nu_1^4.
\]
Subtracting the squared second moment gives
\[
  \operatorname{Var}(a_{t,j})
  = \mathbb{E}\bigl[(G_j+\zeta)^4\bigr]
    - \bigl(\mathbb{E}\bigl[(G_j
    +\zeta)^2\bigr]\bigr)^2
  = 4G_j^2\nu_1^2 + 2\nu_1^4.
\]
By independence of the $T_1$ steps,
\[
  \operatorname{Var}(a_j)
  = \frac{2\nu_1^4 + 4\nu_1^2 G_j^2}{T_1}.
\]

\smallskip
\noindent\textit{Step 2: Misranking bound.}
Suppose $G_p^2 > G_q^2$, and define
\[
  \Delta_a = a_p - a_q,
  \qquad
  \Delta_G = G_p^2 - G_q^2 > 0.
\]
By Step~1 and independence across coordinates,
\begin{equation}
  \label{eq:score-difference-variance}
  \mathbb{E}[\Delta_a] = \Delta_G,
  \qquad
  \operatorname{Var}(\Delta_a)
  = \frac{4\nu_1^4 + 4\nu_1^2(G_p^2 + G_q^2)}{T_1}.
\end{equation}
Since $\Delta_G > 0$, the misranking event
$\{a_p < a_q\} = \{\Delta_a < 0\}$ implies
$|\Delta_a - \Delta_G|
\ge \Delta_G$, so
\begin{align*}
	\Pr[a_p < a_q] = \Pr[\Delta_a < 0] &\le \Pr[|\Delta_a - \Delta_G| \ge \Delta_G] \\
	&\le \frac{\operatorname{Var}(\Delta_a)}{\Delta_G^2} = \frac{4\nu_1^4 + 4\nu_1^2(G_p^2 + G_q^2)}{T_1(G_p^2 - G_q^2)^2}
\end{align*}
where the last inequality is Chebyshev's inequality applied to
$\Delta_a - \Delta_G$ with threshold
$\Delta_G$.
This proves the proposition.
\end{proof}

\begin{remark}[Tightness of the bound]
The Chebyshev bound above gives an $O(T_1^{-1})$ misranking
probability.  Since each $a_{t,j} - G_j^2 = 2G_j\cdot\frac{z_{t,j}}{B} 
+ \left(\frac{z_{t,j}}{B}\right)^2 - \nu_1^2$ is a sub-exponential random variable
(a linear combination of Gaussian and centered chi-squared terms),
Bernstein's inequality applied to $\Delta_a - \Delta_G$ gives an
exponential tail:

\begin{align*}
  \Pr[a_p < a_q] \le \exp\Biggl( -c\,T_1 \min \Biggl\{
  & \frac{\Delta_G^2}{\nu_1^4 + \nu_1^2(G_p^2 + G_q^2)}, \\
  & \frac{\Delta_G}{\nu_1^2 + \max\{|G_p|,|G_q|\}\nu_1} \Biggr\} \Biggr)
\end{align*}

for an absolute constant $c > 0$.  The polynomial Chebyshev bound
is stated in the proposition because it has a simpler closed form;
the exponential bound follows from the same variance computation
via standard sub-exponential concentration.
\end{remark}

\section{Proofs for Convergence Analysis}
\label{app:convergence-proofs}

\subsection{Proof of Theorem~\ref{thm:projected-stationarity}}
\label{app:convergence-proof}

\begin{proof}
The proof is a standard smooth nonconvex descent argument adapted to
masked DP-SGD with an explicit bias term. The key step is a quadratic
upper bound that separates the projected gradient from the clipping bias.

Throughout, all expectations are conditional on the Phase~1 transcript
\(\mathcal{T}_1\), and we write \(\mathbb{E}_2[\cdot]\) accordingly.
Recall the notation from Section~\ref{sec:convergence-setup}:
$
  G_t=\nabla L_D(\theta_t), P_t=P_{A_t}, k_t=|A_t|.
$
For \textsc{TP-TopK},
$A_t = A$ and $k_t = k$ for all $t$, so $\bar{k} = k$; the proof
is stated for a general predictable support schedule $\{A_t\}$.

\noindent\textit{Step 1: One-step descent.}
By \(\beta\)-smoothness of \(L_D\) and the Phase~2 update
\(\theta_{t+1}=\theta_t-\eta\tilde{g}_t^A\) with
\(\tilde{g}_t^A=g_t^A+\zeta_t\),
\[
  L_D(\theta_{t+1})
  \le
  L_D(\theta_t)
  -\eta\langle G_t,\tilde{g}_t^A\rangle
  +\frac{\beta\eta^2}{2}\|\tilde{g}_t^A\|_2^2.
\]

\noindent\textit{Step 2: Conditional expectation of the inner product.}
By the biased-gradient relation in
Section~\ref{sec:convergence-setup},
\[
  \mathbb{E}_2[g_t^A\mid\mathcal{F}_t]=x_t+b_t.
\]
Together with the conditional zero mean of \(\zeta_t\), this gives
\[
  \mathbb{E}_2[\tilde{g}_t^A\mid\mathcal{F}_t] = x_t+b_t.
\]
Since \(b_t=P_tb_t\), both \(x_t\) and \(b_t\) are supported on
\(A_t\), so \(\langle P_{\bar{A}_t}G_t, x_t+b_t\rangle=0\). Therefore,
\[
  \mathbb{E}_2[\langle G_t,\tilde{g}_t^A\rangle\mid\mathcal{F}_t]
  =\langle G_t,x_t+b_t\rangle
  =\langle x_t,x_t+b_t\rangle.
\]

\noindent\textit{Step 3: Key inequality separating gradient from bias.}
The following inequality separates the projected-gradient term
$\|x_t\|_2^2$ from the clipping-bias term $\|b_t\|_2^2$, allowing
the bias to appear as an additive term in the final bound.
We claim that for \(0<\eta\le 1/\beta\),
\begin{equation}
  \label{eq:key-ineq}
  -\langle x_t,x_t+b_t\rangle
  +\frac{\beta\eta}{2}\|x_t+b_t\|_2^2
  \le
  -\frac{1}{2}\|x_t\|_2^2+\frac{1}{2}\|b_t\|_2^2.
\end{equation}
Indeed, \(-\langle x_t,x_t+b_t\rangle=-\|x_t\|_2^2-\langle x_t,b_t\rangle\).
Since \(\beta\eta\le 1\),
\[
  \frac{\beta\eta}{2}\|x_t+b_t\|_2^2
  \le\frac{1}{2}\|x_t+b_t\|_2^2
  =\frac{1}{2}\|x_t\|_2^2+\langle x_t,b_t\rangle+\frac{1}{2}\|b_t\|_2^2.
\]
Adding the two displays gives~\eqref{eq:key-ineq}.

\noindent\textit{Step 4: Bounding the second-moment term.}
Expanding \(\|\tilde{g}_t^A\|_2^2=\|g_t^A\|_2^2
+2\langle g_t^A,\zeta_t\rangle+\|\zeta_t\|_2^2\), the cross
term vanishes after conditioning on \(\mathcal{F}_t\) since
\(\mathbb{E}_2[\zeta_t\mid\mathcal{F}_t]=0\).

For the $\|g_t^A\|_2^2$ term, we apply the
bias--variance decomposition, where $\|\mathbb{E}_2[g_t^A\mid\mathcal{F}_t]\|_2^2 
= \|x_t+b_t\|_2^2$ by definition of $x_t$ and $b_t$:
\begin{align*}
  \mathbb{E}_2[\|g_t^A\|_2^2 \mid \mathcal{F}_t]
  &= \mathbb{E}_2\!\left[
       \|g_t^A - \mathbb{E}_2[g_t^A\mid\mathcal{F}_t]\|_2^2
     \mid \mathcal{F}_t\right]
   + \|\mathbb{E}_2[g_t^A\mid\mathcal{F}_t]\|_2^2 \\
  &\le \frac{\tau^2}{B} + \|x_t+b_t\|_2^2,
\end{align*}
where the inequality uses Assumption~\ref{ass:variance} for
the first term and the relation
$\mathbb{E}_2[g_t^A\mid\mathcal{F}_t]=x_t+b_t$ for the
second.

Since $\zeta_t = \frac{1}{B}P_t z_t$ with
$z_t \sim \mathcal{N}(0,\sigma_2^2C_2^2I_d)$
and $P_t$ projecting onto the $k_t$ active coordinates,
\[
  \mathbb{E}_2[\|\zeta_t\|_2^2 \mid \mathcal{F}_t]
  = \frac{1}{B^2}\mathbb{E}_2[\|P_t z_t\|_2^2 \mid \mathcal{F}_t]
  = \frac{k_t \sigma_2^2 C_2^2}{B^2}
  = k_t \nu_2^2,
\]
where the second equality uses the fact that $P_t z_t$ has
exactly $k_t$ nonzero coordinates each distributed as
$\mathcal{N}(0,\sigma_2^2C_2^2)$, independently of
$\mathcal{F}_t$ conditional on $\mathcal{T}_1$.

Combining the two bounds,
\[
  \mathbb{E}_2[\|\tilde{g}_t^A\|_2^2\mid\mathcal{F}_t]
  \le\|x_t+b_t\|_2^2+\frac{\tau^2}{B}+k_t\nu_2^2.
\]

\noindent\textit{Step 5: Combining the descent terms.}
Taking conditional expectation given \(\mathcal{F}_t\) and
applying Steps~2 and~4 to Step~1,

\begin{align*}
  \mathbb{E}_2[L_D(\theta_{t+1}) \mid \mathcal{F}_t]
  &\le L_D(\theta_t) - \eta\langle x_t, x_t+b_t\rangle \\
  &\quad + \frac{\beta\eta^2}{2} \left( \|x_t+b_t\|_2^2
  + \frac{\tau^2}{B} + k_t \nu_2^2 \right).
\end{align*}

Separating the $\|x_t+b_t\|_2^2$ term and
factoring out $\eta$,
\begin{align*}
  \mathbb{E}_2[L_D(\theta_{t+1}) \mid \mathcal{F}_t]
  &\le L_D(\theta_t)
  + \eta\!\left(
      -\langle x_t,\, x_t+b_t\rangle
      + \frac{\beta\eta}{2}\|x_t+b_t\|_2^2
    \right) \\
  &\quad + \frac{\beta\eta^2}{2}\!\left(
      \frac{\tau^2}{B} + k_t\nu_2^2
    \right).
\end{align*}
Applying~\eqref{eq:key-ineq} to the parenthesised term,

\[
  \mathbb{E}_2[L_D(\theta_{t+1})\mid\mathcal{F}_t]
  \le L_D(\theta_t)
  -\frac{\eta}{2}\|x_t\|_2^2
  +\frac{\eta}{2}\|b_t\|_2^2
  +\frac{\beta\eta^2}{2}\!\left(
      \frac{\tau^2}{B}+k_t\nu_2^2
  \right).
\]

\noindent\textit{Step 6: Telescoping.}
Rearranging the previous inequality gives
\[
  \frac{\eta}{2}\|x_t\|_2^2
  \le
  L_D(\theta_t)
  -
  \mathbb E_2[L_D(\theta_{t+1})\mid\mathcal F_t]
  +
  \frac{\eta}{2}\|b_t\|_2^2
  +
  \frac{\beta\eta^2}{2}
  \left(
    \frac{\tau^2}{B}
    +
    k_t\nu_2^2
  \right).
\]

Taking total expectation and summing over \(t=0,\ldots,T_2-1\),

\begin{align*}
  \frac{\eta}{2} \sum_{t=0}^{T_2-1} \mathbb{E}_2\|x_t\|_2^2 
  &\le L_0 - \mathbb{E}_2[L_D(\theta_{T_2})] 
  + \frac{\eta}{2} \sum_{t=0}^{T_2-1} \mathbb{E}_2\|b_t\|_2^2 \\
  &\quad + \frac{\beta\eta^2}{2} \sum_{t=0}^{T_2-1} 
  \left( \frac{\tau^2}{B} + \mathbb{E}_2[k_t]\, \nu_2^2 \right),
\end{align*}
where $L_0 = L_D(\theta^{(1)})$ is
the Phase~2 initial loss and $x_t = P_t\nabla L_D(\theta_t)$
as defined in Section~\ref{sec:convergence-setup}.
Since $L_D(\theta_{T_2})\ge L^*$, dividing both sides by
$\eta T_2/2$ and using
$\bar{k} = T_2^{-1}\sum_{t=0}^{T_2-1}\mathbb{E}_2[k_t]$
yields~\eqref{eq:projected-bound}.
\end{proof}

\subsection{Proof of Lemma~\ref{lem:clipping-bias}}
\label{app:proof-clipping-bias}

\begin{proof}
Let $r_{t,i} = g^A_{t,i} - \bar{g}^A_{t,i}$ be the clipping residual. By definition,
$b_t = \mathbb{E}_2[g_t^A \mid \mathcal{F}_t] - P_t G_t$.
Since $g^A_{t,i} = m\odot\nabla\ell(\theta_t;x_i,y_i)$ and the 
mini-batch is drawn i.i.d.\ from $D$,
$\mathbb{E}_2[g^A_{t,i}\mid\mathcal{F}_t] = m\odot\nabla L_D(\theta_t) = P_tG_t$.
Therefore
\[
  b_t
  = \mathbb{E}_2[\bar{g}^A_{t,i} \mid \mathcal{F}_t]
    - \mathbb{E}_2[g^A_{t,i} \mid \mathcal{F}_t]
  = -\mathbb{E}_2[r_{t,i} \mid \mathcal{F}_t].
\]
Jensen's inequality gives
\[
  \|b_t\|_2
  \le
  \mathbb{E}_2\!\left[\|r_{t,i}\|_2 \mid \mathcal{F}_t\right].
\]
By the definition of $\ell_2$-clipping,
\[
  \|r_{t,i}\|_2
  =
  \bigl(\|g^A_{t,i}\|_2 - C_2\bigr)_+.
\]
This proves the first claim. If
$\|g^A_{t,i}\|_2 \le C_2$ almost surely, then
$r_{t,i} = 0$ almost surely, and hence $b_t = 0$.
\end{proof}

\subsection{Proof of Corollary~\ref{cor:full-gradient}}
\label{app:proof-full-gradient}

\begin{proof}
By Definition~\ref{def:energy-capture},
\[
  \frac{1}{T_2}\sum_{t=0}^{T_2-1}\mathbb{E}_2\|P_t\nabla L_D(\theta_t)\|_2^2
  =
  \alpha_T \cdot
  \frac{1}{T_2}\sum_{t=0}^{T_2-1}\mathbb{E}_2\|\nabla L_D(\theta_t)\|_2^2.
\]
Theorem~\ref{thm:projected-stationarity} bounds the left-hand side:
\begin{align*}
  \alpha_T \cdot \frac{1}{T_2} \sum_{t=0}^{T_2-1} \mathbb{E}_2\|\nabla L_D(\theta_t)\|_2^2 
  &\le \frac{2(L_0-L^*)}{\eta T_2} + \frac{1}{T_2} \sum_{t=0}^{T_2-1} \mathbb{E}_2\|b_t\|_2^2 \\
  &\quad + \beta\eta \left( \frac{\tau^2}{B} + \bar{k}\,\nu_2^2 \right).
\end{align*}
Dividing both sides by $\alpha_T > 0$ gives the stated bound.
\end{proof}


\bibliographystyle{ACM-Reference-Format}
\bibliography{sample}

\end{document}